\PassOptionsToPackage{table}{xcolor}
\documentclass[manuscript,screen,nonacm]{acmart}

\usepackage{algorithm}
\usepackage{algpseudocode}
\AtBeginDocument{%
  }

\setcopyright{none}
\ifPDFTeX
\usepackage[utf8]{inputenc}
\usepackage[T1]{fontenc}
\fi
\usepackage{microtype}
\usepackage{amsmath,amsfonts,amsthm}
\usepackage{mathtools}
\ifPDFTeX
\usepackage{bm}
\else
\providecommand{\bm}[1]{\symbfit{#1}}
\fi
\usepackage{dsfont}
\usepackage{nicefrac}
\usepackage{subfigure}
\usepackage{algorithm}
\usepackage{algpseudocode}

\usepackage{graphicx}
\usepackage{booktabs}
\usepackage{array}
\usepackage{multirow}
\usepackage{xcolor}
\usepackage{caption}
\usepackage{siunitx}

\usepackage{amsmath,amsfonts}
\usepackage{xcolor}
\usepackage{siunitx}
\usepackage{moreverb}
\usepackage{mathtools}
\usepackage{comment}
\usepackage{graphicx}
\usepackage{tikz}
\usepackage{subfigure}
\usepackage{framed}
\usepackage{matlab-prettifier}
\usepackage{booktabs} 
\usepackage{array}
\usepackage{colortbl} 
\usepackage{hyperref}
\usepackage{listings}
\usepackage{enumitem}
\usepackage{threeparttable}
\usepackage{multirow}
\usepackage{fancyvrb}
\usepackage{varwidth}

\usepackage{tikz}
\usetikzlibrary{
    shapes,
    shapes.geometric,
    arrows,
    arrows.meta,
    positioning,
    fit,
    backgrounds,
    shadows.blur,
    calc
}

\usepackage{xurl}
\usepackage{doi}

\hypersetup{
    colorlinks=true,
    linkcolor=blue!50!black,
    citecolor=blue!50!black,
    urlcolor=blue!50!black
}

\usepackage{comment}
\usepackage{pifont}
\usepackage{lipsum}

\theoremstyle{plain}
\newtheorem{theorem}{Theorem}[section]

\theoremstyle{definition}

\theoremstyle{remark}
\newtheorem{remark}{Remark}

\newcommand{\mathcircled}[1]{%
    \tikz[baseline=(char.base)]{
        \node[shape=circle,draw,inner sep=1pt] (char) {$#1$};
    }%
}

\definecolor{matrixcolor}{RGB}{200,240,200}
\definecolor{featurecolor}{RGB}{230,220,240}
\definecolor{mlpcolor}{RGB}{180,240,250}
\definecolor{outputcolor}{RGB}{250,220,220}

\tikzset{
    block/.style={
        rectangle,
        draw=#1,
        fill=#1!10,
        thick,
        rounded corners=8pt,
        minimum width=3.2cm,
        minimum height=4.2cm,
        font=\sffamily\small,
        blur shadow={shadow blur steps=5},
        align=center
    },
    arrow/.style={
        -Stealth,
        thick,
        shorten >=3pt,
        shorten <=3pt
    },
    label/.style={
        font=\sffamily\scriptsize,
        text=black!70
    },
    neuron/.style={
        circle,
        draw=#1!80,
        fill=#1!20,
        thick,
        minimum size=6pt,
        inner sep=0pt
    },
    connection/.style={
        -,
        draw=black!20,
        thin
    }
}

\begin{document}

\title{Estimating Condition Numbers with Graph Neural Networks}

\author{Erin Carson}
\thanks{Supported by the Charles University Research Centre program No. UNCE/24/SCI/005 and the European Union (ERC, inEXASCALE, 101075632). Views and opinions expressed are those of the authors only and do not necessarily reflect those of the European Union or the European Research Council.}
\email{carson@karlin.mff.cuni.cz}
\affiliation{%
  \institution{Department of Numerical Mathematics, Charles University}
 \city{Prague}
  \country{Czech Republic}
}

\author{Xinye Chen}

\email{xinye.chen@lip6.fr}
\affiliation{%
  \institution{LIP6, Sorbonne Université, CNRS}
 \city{Paris}
  \country{France}
}

\renewcommand{\shortauthors}{Carson and Chen}

\begin{abstract}
In this paper, we propose a fast method for estimating the condition number of sparse matrices using graph neural networks (GNNs). For efficient deployment of GNNs, we introduce a graph feature construction with $\mathrm{O}(\mathrm{nnz} + n)$ complexity, where $\mathrm{nnz}$ is the number of non-zero elements in the matrix and $n$ denotes the matrix dimension. We propose two schemes for estimating the matrix condition number using GNNs; one follows by decomposing the condition number and predicts the relatively more computationally intensive part $\|\mathbf{A}^{-1}\|$, without explicitly forming the inverse, while the other is to predict the whole condition number $\kappa$. Our approach can be extended to an arbitrary norm. Extensive experiments are conducted for the estimation of the 1-norm and 2-norm condition numbers, which show that our method achieves a significant speedup over the traditional numerical estimation methods. Our software for GNN condition number estimator is made publicly available at \url{https://github.com/inEXASCALE/sparse-kappa}.
\end{abstract}




\maketitle

\section{Introduction}
\label{sec:introduction}
 
The paper is concerned with estimating the  condition number of a sparse matrix $A$ of arbitrary size. The condition number $\kappa(\mathbf{A})$ of a matrix $\mathbf{A} \in \mathbb{R}^{n \times n}$ quantifies the sensitivity of the solution of a linear system to perturbations in the input data. For a nonsingular matrix $\mathbf{A}$, the condition number with respect to the $p$-norm is defined by
\begin{equation}
\kappa_p(\mathbf{A}) = \|\mathbf{A}\|_p \, \|\mathbf{A}^{-1}\|_p,
\label{eq:condition_number}
\end{equation}
where the applications of interest lie in $p=1, 2, +\infty$.  Calculating the condition number exactly is generally prohibitively expensive for large and sparse matrices. In practice, we rarely do this because it often involves costly steps, such as finding the largest and smallest singular values or computing the matrix inverse, which do not scale well for large, sparse problems. For instance, the spectral condition number $\kappa_2(\mathbf{A})$ depends on the extreme singular values of $\mathbf{A}$. In dense linear algebra, these may be obtained by a full singular value decomposition (SVD), but this takes $\mathcal{O}(n^3)$ floating-point operations. Similarly, forming $\mathbf{A}^{-1}$ directly also takes $\mathcal{O}(n^3)$ work in the dense case and is usually avoided. With sparse matrices, things get even trickier because sparse factorizations can lead to significant fill-in.

To address the high computational cost, one often resorts to estimation methods. As referenced in \cite[Ch. 15]{higham2002accuracy}, condition number estimation is the process of calculating an inexpensive estimate of a condition number, where inexpensive usually means that the cost is an order of magnitude less than the exact method.  For spectral condition number estimation, iterative methods such as power iteration, Lanczos, Arnoldi, or Golub--Kahan bidiagonalization reduce the cost substantially, typically requiring $\mathcal{O}(k \cdot \mathrm{nnz}(\mathbf{A}))$ work for $k$ iterations, up to orthogonalization costs. For decades, efficient condition number estimation has therefore been recognized as essential in assessing the stability of linear systems and least-squares problems, particularly in conjunction with matrix factorizations such as LU with partial pivoting; early practical strategies focused on inexpensive post-factorization estimates that avoid explicit inversion~\cite{hager1984condition,higham1988estimating}.

For large sparse matrices, the exact computation of the condition number is often impractical because evaluating $\|\mathbf{A}^{-1}\|_p$ exactly would, in the most direct approach, require either forming $\mathbf{A}^{-1}$ explicitly or solving $n$ linear systems to recover its columns. By contrast, the matrix norm $\|\mathbf{A}\|_p$ is relatively inexpensive to compute, e.g., for 1-norm:
\begin{equation}
\|\mathbf{A}\|_1 = \max_{1 \le j \le n} \sum_{i=1}^n |a_{ij}|,
\end{equation}
which can be evaluated in $\mathcal{O}(\mathrm{nnz}(\mathbf{A}))$ time for a sparse matrix.

Accordingly, the main difficulty  of estimating $\kappa_p(\mathbf{A}) = \|\mathbf{A}\|_p \, \|\mathbf{A}^{-1}\|_p$ often lies in estimating $\|\mathbf{A}^{-1}\|_p$. For matrices small enough to be treated by dense linear algebra routines, this quantity may be computed directly from a dense factorization. For instance, a standard alternative to estimate $\|\mathbf{A}^{-1}\|_1$ is via the Hager--Higham iterative algorithm~\cite{hager1984condition,higham1988estimating}. This procedure avoids explicit inversion by solving a sequence of linear systems
\begin{equation}
\mathbf{A} x^{(k)} = \mathbf{b}^{(k)},
\end{equation}
where the right-hand sides $\mathbf{b}^{(k)}$ are chosen adaptively from the sign pattern of intermediate iterates so as to approximate a vector attaining the maximum in the induced $1$-norm. In practice, a sparse LU factorization of $\mathbf{A}$ is computed once and then reused for all subsequent triangular solves, making the estimator substantially cheaper than explicit inversion.

However, the aforementioned approximate method for condition number estimation mentioned above is limited to the 1-norm case. For the 2-norm (spectral) condition number, which requires estimating the ratio of the largest to smallest singular values, following
\begin{equation}
\kappa_2(\mathbf{A}) := \|\mathbf{A}\|_2 \, \|\mathbf{A}^{-1}\|_2
= \frac{\sigma_{\max}(\mathbf{A})}{\sigma_{\min}(\mathbf{A})},
\end{equation}
where one often uses iterative methods such as the power method or Lanczos algorithm (or its bidiagonalization variant) to approximate the extreme singular values \cite{golub2013matrix} and  the main challenge is to calculate the smallest singular value $\sigma_{\min}$. However, direct computation through a full singular value decomposition is generally infeasible for truly large sparse matrices. Both $\|A\|_2$ and $\|A^{-1}\|_2$ can be estimated via the power method \cite{10.5555/59657}. An alternative practical way is to exploit Krylov subspace projections, which reduce the high-dimensional problem to a much smaller one while still capturing the behavior of the extreme singular values with good accuracy after a modest number of steps.

In the general nonsymmetric setting, a natural tool for estimating the extreme singular values is Golub--Kahan bidiagonalization~\cite{golub2013matrix}. Starting from an initial unit vector, the process constructs orthonormal bases
\begin{equation*}
\mathbf{U}_{k+1} = [\mathbf{u}_1,\dots,\mathbf{u}_{k+1}] \in \mathbb{R}^{n \times (k+1)},
\qquad
\mathbf{V}_{k} = [\mathbf{v}_1,\dots,\mathbf{v}_{k}] \in \mathbb{R}^{n \times k},
\end{equation*}
together with a lower bidiagonal matrix $\mathbf{B}_k \in \mathbb{R}^{(k+1)\times k}$ such that
\begin{equation}
\mathbf{A}\mathbf{V}_k = \mathbf{U}_{k+1}\mathbf{B}_k,
\qquad
\mathbf{A}^\top \mathbf{U}_{k+1}
= \mathbf{V}_k \mathbf{B}_k^\top + \alpha_{k+1}\mathbf{v}_{k+1}\mathbf{e}_{k+1}^\top, 
\label{eq:gk_relation}
\end{equation}
and the projected matrix $\mathbf{B}_k$ has bidiagonal structure,
\begin{equation}
\mathbf{B}_k =
\begin{pmatrix}
\alpha_1 &        &        &        \\
\beta_2  & \alpha_2 &      &        \\
         & \beta_3  & \ddots &      \\
         &          & \ddots & \alpha_k \\
         &          &        & \beta_{k+1}
\end{pmatrix}
\in \mathbb{R}^{(k+1)\times k},
\end{equation}
where the coefficients $\alpha_j > 0$ and $\beta_j \ge 0$ arise from the recurrence. Each iteration requires one multiplication by $\mathbf{A}$ and one multiplication by $\mathbf{A}^\top$, so the dominant cost is typically $\mathcal{O}(\mathrm{nnz}(\mathbf{A}))$ per step, in addition to orthogonalization. Accordingly, the singular values of the small bidiagonal matrix $\mathbf{B}_k$ provide Ritz-type approximations to the singular values of $\mathbf{A}$. In particular, the largest singular value usually converges first, while the smallest singular value may require a somewhat larger subspace dimension, especially when it is poorly separated. Once $k$ is sufficiently large, one obtains the estimate
\begin{equation}
\widehat{\kappa}_2(\mathbf{A})
\approx
\frac{\sigma_{\max}(\mathbf{B}_k)}{\sigma_{\min}(\mathbf{B}_k)}.
\end{equation}
In practice, convergence may be monitored by the stabilization of the extreme singular-value approximations or by residual-based stopping criteria. In the special case where $\mathbf{A}$ is symmetric positive definite, the situation simplifies considerably by replacing $\sigma_{\max}(\mathbf{B}_k)$ and $\sigma_{\min}(\mathbf{B}_k)$ with $\lambda_{\max}(\mathbf{A})$ and $\lambda_{\min}(\mathbf{A})$, respectively. As such, the Lanczos algorithm may be applied directly to $\mathbf{A}$, producing a symmetric tridiagonal matrix $\mathbf{T}_k$ whose extreme eigenvalues approximate those of $\mathbf{A}$. The spectral condition number is then estimated by the ratio of the largest to the smallest Ritz eigenvalue of $\mathbf{T}_k$.

For a general nonsymmetric matrix, however, Arnoldi iteration applied directly to $\mathbf{A}$ approximates eigenvalues rather than singular values, and these are not sufficient in general to estimate $\kappa_2(\mathbf{A})$, since
\begin{equation*}
\|\mathbf{A}\|_2 \neq \max_i |\lambda_i(\mathbf{A})|
\quad\text{and}\quad
\|\mathbf{A}^{-1}\|_2 \neq \frac{1}{\min_i |\lambda_i(\mathbf{A})|}
\end{equation*}
for a nonnormal matrix. Therefore, for the estimation of the spectral condition number in the nonsymmetric case, one typically works with singular-value-oriented procedures such as Golub--Kahan bidiagonalization, or alternatively with Krylov subspace methods applied to the normal equations $\mathbf{A}^\top \mathbf{A}$, bearing in mind that the latter squares the condition number and may worsen numerical sensitivity.

\begin{figure}[htbp]
\centering
\resizebox{\textwidth}{!}{%
\begin{tikzpicture}[
    node distance=2.5cm,
    every node/.style={align=center},
    block/.style={rectangle, draw, rounded corners=5pt, text centered, minimum height=4em, minimum width=6em, inner sep=10pt, font=\sffamily},
    arrow/.style={thick,->,>=stealth,font=\sffamily\scriptsize},
    matrixcolor_node/.style={block, fill=matrixcolor!80, draw=matrixcolor},
    featurecolor_node/.style={block, fill=featurecolor!80, draw=featurecolor},
    gnn_node/.style={block, fill=mlpcolor!80, draw=mlpcolor, minimum height=12em}, 
    outputcolor_node/.style={block, fill=outputcolor!80, draw=outputcolor},
    gnode/.style={circle, draw, fill=blue!10, minimum size=1.2em},         
    gnode_hidden/.style={circle, draw, fill=blue!40, minimum size=1.2em},  
    gnode_final/.style={circle, draw, fill=blue!70, minimum size=1.8em}    
]

\node[matrixcolor_node] (matrix) {
    \textbf{Sparse Matrix} \\[0.4cm]
    $\mathbf{A} \in \mathbb{R}^{n \times n}$ \\[0.2cm]
    $\text{nnz}(\mathbf{A}) = m$ \\[0.4cm]
    \begin{tikzpicture}[scale=0.25]
        \foreach \i in {0,...,4} {
            \draw[black!30] (0,\i) -- (5,\i);
            \draw[black!30] (\i,0) -- (\i,5);
        }
        \fill[black!80] (0.5,4.5) circle (4pt);
        \fill[black!80] (1.5,4.5) circle (4pt);
        \fill[black!80] (1.5,3.5) circle (4pt);
        \fill[black!80] (0.5,3.5) circle (4pt);
        \fill[black!80] (2.5,3.5) circle (4pt);
        \fill[black!80] (2.5,2.5) circle (4pt);
        \fill[black!80] (3.5,2.5) circle (4pt);
        \fill[black!80] (1.5,1.5) circle (4pt);
        \fill[black!80] (3.5,1.5) circle (4pt);
        \fill[black!80] (4.5,0.5) circle (4pt);
    \end{tikzpicture}
};

\node[above=0.25cm of matrix, font=\sffamily\bfseries\large] (title_matrix) {(1) Input};

\node[featurecolor_node, right=of matrix] (features) {
    \textbf{Feature Vector} \\[0.8cm]
    {\Large $\bm{\phi}(\mathbf{A}) \in \mathbb{R}^{k}$} \\[0.8cm]
    $\mathcal{O}(\text{nnz}(\mathbf{A}))$
};

\node[above=0.25cm of features, font=\sffamily\bfseries\large] (title_features) {(2) Extraction};

\node[gnn_node, right=of features] (gnn_block) {
    \textbf{Graph Neural Network} \\[0.6cm]
    \begin{tikzpicture}[scale=0.85, transform shape]
        
        \begin{scope}[local bounding box=g1]
            \node[gnode] (n1_1) at (0, 0.8) {};
            \node[gnode] (n1_2) at (-0.8, -0.2) {};
            \node[gnode] (n1_3) at (0.8, -0.2) {};
            \node[gnode] (n1_4) at (0, -1) {};
            \node[gnode] (n1_5) at (1.6, 0.9) {};
            
            \draw (n1_1) -- (n1_2); \draw (n1_1) -- (n1_3); \draw (n1_1) -- (n1_5);
            \draw (n1_2) -- (n1_4); \draw (n1_3) -- (n1_4); \draw (n1_3) -- (n1_5);
        \end{scope}
        
        \draw[->, >=stealth, thick] (g1.east) ++(-0.3,0) -- ++(0.7,0) coordinate (arr1);
        
        \begin{scope}[local bounding box=g2, xshift=4cm]
            \node[gnode_hidden] (n2_1) at (0, 0.8) {};
            \node[gnode_hidden] (n2_2) at (-0.8, -0.2) {};
            \node[gnode_hidden] (n2_3) at (0.8, -0.2) {};
            \node[gnode_hidden] (n2_4) at (0, -1) {};
            \node[gnode_hidden] (n2_5) at (1.6, 0.9) {};
            
            \draw (n2_1) -- (n2_2); \draw (n2_1) -- (n2_3); \draw (n2_1) -- (n2_5);
            \draw (n2_2) -- (n2_4); \draw (n2_3) -- (n2_4); \draw (n2_3) -- (n2_5);
        \end{scope}

        \draw[->, >=stealth, thick] (g2.east) ++(0.2,0) -- ++(0.8,0) coordinate (arr2);

        \node[gnode_final, right=0.2cm of arr2] (readout) {};

    \end{tikzpicture}
};

\node[above=0.25cm of gnn_block, font=\sffamily\bfseries\large] (title_gnn) {(3) Model Training};

\node[outputcolor_node, right=of gnn_block] (output) {
    \textbf{Estimate} \\[0.8cm]
    {\fontsize{60pt}{60pt}\selectfont $\hat{\kappa}$}
};

\node[above=0.25cm of output, font=\sffamily\bfseries\large] (title_output) {(4) Result};

\draw[arrow] (matrix) -- node[midway, above] {Extract} (features);
\draw[arrow] (features) -- node[midway, above] {Forward} (gnn_block);
\draw[arrow] (gnn_block) -- node[midway, above] {Predict} node[midway, below] {$\log_{10}\kappa_1$} (output);

\begin{pgfonlayer}{background}
    \node[
        fill=pink!8,
        rounded corners=10pt,
        fit=(matrix)(title_matrix),
        inner sep=10pt
    ] {};
    
    \node[
        fill=pink!8,
        rounded corners=10pt,
        fit=(features)(title_features),
        inner sep=10pt
    ] {};
    
    \node[
        fill=pink!8,
        rounded corners=10pt,
        fit=(gnn_block)(title_gnn),
        inner sep=10pt
    ] {};
    
    \node[
        fill=pink!8,
        rounded corners=10pt,
        fit=(output)(title_output),
        inner sep=10pt
    ] {};
\end{pgfonlayer}

\end{tikzpicture}
}
\caption{Pipeline for fast condition number estimation. The method consists of four stages: (1) input sparse matrix $\mathbf{A}$, (2) extract matrix-theoretic features, (3) model training using a multilayer perceptron, and (4) output the estimated condition number $\hat{\kappa}(\mathbf{A})$.}
\label{fig:method}
\end{figure}

We propose a data-driven approach based on GNNs for condition number estimation for sparse systems that: (i) Extracts structural and numerical features of matrix data in $\mathcal{O}(\text{nnz})$ time; (ii) Trains a deep neural network to predict $\log_{10} \kappa(A)$; (iii) Achieves inference time scales insensitive to, or independent of matrix size. Additionally, a novel graph machine learning approach is used for rapidly estimating the arbitrary-norm condition number of large sparse matrices arising in scientific computing applications. Unlike traditional iterative methods such as the Hager-Higham method \cite{hager1984condition, Higham1987}, our method uses a graph neural network trained on a comprehensive set of matrix-theoretic features to achieve sub-millisecond inference with controllable accuracy. Numerical experiments demonstrate that our approach achieves a significant speedup over classical methods and sparse solvers while maintaining an acceptable relative error. To the best of our knowledge, this is the first work on leveraging graph learning techniques to address condition number estimation. We believe this work can offer unique insights into the application of AI for numerical methods.  We demonstrate the effectiveness of our approach with a simple neural network architecture and challenging condition number estimation tasks. 

The rest of the paper is structured as follows; Section~\ref{sec:related} discusses related work on condition number estimation; Section~\ref{sec:method} formulates our methodologies, with a focus on the matrix features as part of graph construction,  and model architecture and also provide theoretical support for the condition number estimation; Section~\ref{sec:exps} presents our data generation routines for training and analyzes the simulations on performance; Section~\ref{sec:limitations} illustrates the limitations of our approach, which we leave as future work; Section~\ref{sec:concludion} concludes the paper.

\section{Related Work}\label{sec:related}
Recent advances in artificial intelligence have reshaped the landscape of scientific computing, motivating the integration of learning-based models with simulation, numerical linear algebra, and high-performance computing workflows~\cite{doi:10.1126/science.aef4214}. Deep learning methods have been increasingly explored for scientific computing tasks such as surrogate modeling, operator learning, and solver acceleration~\cite{karniadakis2021physics, li2020fourier}. Motivated by this broader trend, we study the use of graph neural networks to estimate matrix condition numbers for sparse linear systems. 

Estimating the condition number is critical for a linear matrix solve. A well-conditioned linear system generally does not require preconditioning, whereas a poorly conditioned one often does. In addition, a well-conditioned linear system can often be solved with reduced precision in certain steps to improve performance. 

The foundational contribution to estimating the condition number in this line of work is due to Cline et al. \cite{Cline1979}, who introduced a simple yet effective estimator based on the solution of two carefully chosen linear systems. Their approach, which was incorporated into the LINPACK library, produces a lower bound on the one-norm condition number at the cost of only a pair of triangular solves. Although computationally attractive, the method may underestimate the true condition number by one or more orders of magnitude in certain pathological cases.

To estimate matrix condition number, one often employs the power method to compute the spectral norm of a matrix, as introduced by Boyd in 1974 \cite{BOYD197495}.  Following this, a significant improvement appeared shortly thereafter when Hager~\cite{hager1984condition} reformulated the one-norm estimation problem as a convex optimization task. By interpreting $\|\mathbf{A}^{-1}\|_1$ as the maximum of $\|A^{-1}x\|_1$ over the unit ball in the one-norm and deriving an efficient iterative procedure (a specialized variant of the simplex method), Hager's algorithm typically produces substantially tighter and more reliable estimates than the earlier heuristic, while still requiring only a modest number of matrix-vector products after factorization.

Higham~\cite{Higham1987} later provided a comprehensive survey of condition-number estimation specifically for triangular matrices---the form that naturally arises after LU factorization. The survey systematically compares growth-factor bounds, residual-based techniques, and inverse iteration approaches, offering both theoretical insights and practical recommendations regarding the trade-off between accuracy and computational cost. This work serves as a unifying reference and highlights subtle numerical issues that must be addressed when implementing estimators for upper- or lower-triangular factors.

Building directly on Hager's optimization framework, Higham~\cite{higham1988estimating} subsequently proposed a refined one-norm estimator that addresses several practical limitations of the original algorithm. In particular, Higham provided a rigorous convergence analysis, extended the method to complex matrices, introduced safeguards against premature termination, and developed improved stopping criteria that enhance robustness without sacrificing efficiency. These theoretical and algorithmic refinements established a more reliable foundation for condition-number estimation in floating-point arithmetic.

For practical use, Higham~\cite{higham1988fortran} later released portable Fortran~77 implementations (subsequently published as Algorithm~674 in~\cite{10.1145/63522.214391}). These routines implement the improved one-norm estimator together with its application to condition-number computation, include careful handling of both real and complex arithmetic, and incorporate extensive testing and error-handling mechanisms. The routines were later adopted---with minor modifications---into LAPACK and remain the \emph{de facto} standard in many modern numerical libraries, effectively bridging the gap between theoretical development and robust, widely used software.

For the 2-norm condition number estimation, Krylov subspace methods \cite{lanczos1950iteration, arnoldi1951principle, saad2003iterative} provide a coherent and computationally efficient framework for condition-number estimation in large sparse problems. Their cost is dominated by sparse matrix--vector multiplications, orthogonalization, and the solution of small projected eigenvalue or singular-value problems, making them far more practical than explicit inversion or full dense decompositions. Based on a randomized method, \cite{Avron2019} solves the spectral condition number problem of a real matrix via estimating the smallest singular value by solving the least -squares problem using LSQR \cite{10.1145/355984.355989}.

There is existing work \cite{hausner2024neural, pmlr-v265-hausner25a} that focuses on using GNNs to generate and select preconditioners for sparse linear systems, achieving excellent performance. The improved performance by GNNs often requires well-designed handcrafted features. However, existing research does not provide theoretical insights into feature modeling and selection; rather, it focuses on interpretation based on experience.  In Section~\ref{sec:method}, we offer theoretical insight why we choose these matrix features based on how it effects the matrix condition number range.

Research on predicting condition numbers using machine learning methods remains relatively limited. The only one that published which we are aware is \cite{doi:10.1137/1.9781611972757.47}, which proposes an approach for predicting the condition number of sparse matrices using support vector regression. Although this method achieves lower accuracy than direct computation or classical estimation techniques, it can still be sufficient for coarse classification tasks, such as determining whether a matrix is well-conditioned or ill-conditioned.

\section{Methodology}\label{sec:method}

Consider the supervised learning problem: given a training set $\mathcal{D} = \{(\mathbf{A}_i, \kappa^i)\}_{i=1}^N$ where $\mathbf{A}_i \in \mathbb{R}^{n_i \times n_i}$ are sparse matrices and $\|\cdot\|$ denotes arbitrary norm by removing subsript and same for $\kappa$, we seek to construct a rapid estimator by exploiting the factorization
\begin{equation}
\kappa(\mathbf{A}) = \|\mathbf{A}\| \cdot \|\mathbf{A}^{-1}\|.
\label{eq:cond_factorization}
\end{equation}

Observing that $\|\mathbf{A}\|$ is less computationally expensive, e.g., $\|\mathbf{A}\|_1$ can be computed exactly in $\mathcal{O}(\mathrm{nnz}(\mathbf{A}))$ time. Therefore, we reformulate the estimation task: rather than directly predicting $\kappa(\mathbf{A})$, we learn a surrogate function $g: \mathcal{M} \to \mathbb{R}_+$ such that
\begin{equation}\label{eq:learning_objective_hybrid}
g(\mathbf{A}) \approx \|\mathbf{A}^{-1}\| \quad \text{or} \quad g(\mathbf{A}) \approx \kappa(\mathbf{A})  \quad  \forall \mathbf{A} \in \mathcal{M},
\end{equation}
where $\mathcal{M}$ denotes the space of sparse matrices arising in target applications. The condition number estimate is then recovered via the hybrid formula
\begin{equation}
\hat{\kappa}(\mathbf{A}) = \|\mathbf{A}\| \cdot g(\mathbf{A}) \quad \text{or} \quad \hat{\kappa}(\mathbf{A}) =g(\mathbf{A}).
\label{eq:hybrid_estimate}
\end{equation}

Following \cite{10.1145/63522.214391}, this scheme is also useful for solving Sylvester equations $AX + XB = C$, where $A$ is an $m$ by $m$ matrix and $ B$ is an $n$ by $n$ matrix. The linear equations may also be presented in the form of $Px = c$, where $P =I_n  \otimes  A + B^T  \otimes  I_m$ as a Kronecker sum. Following \cite{1102170} and \cite{10.1093/imanum/2.3.303}, the estimation of the norm of $P^{-1}$ is beneficial for assessing the accuracy of a computed solution.

\begin{remark}
This factorization-based approach offers three key advantages: (i) the learning target $\|\mathbf{A}^{-1}\|$ exhibits smaller dynamic range and improved numerical stability compared to the full condition number; (ii) the exact computation of $\|\mathbf{A}\|$ eliminates one source of approximation error; and (iii) the decomposition aligns with the structure of classical iterative refinement algorithms.
\end{remark}

To stabilize numerical conditioning and improve learning dynamics, we adopt the logarithmic transformation
\begin{equation}
\tilde{g}(\mathbf{A}) = \log_{10} g(\mathbf{A}),
\label{eq:log_transform_hybrid}
\end{equation}
So the training target becomes 
\begin{equation*}
\tilde{\nu}_i =
\begin{cases} 
\log_{10} \|\mathbf{A}_i^{-1}\| = \log_{10} \kappa_i - \log_{10} \|\mathbf{A}_i\|, & \text{if } \hat{\kappa}(\mathbf{A}) = \|\mathbf{A}\| \cdot g(\mathbf{A}) \\[2mm]
\log_{10} \|\mathbf{A}_i\| \|\mathbf{A}_i^{-1}\|, & \text{if } \hat{\kappa}(\mathbf{A}) = g(\mathbf{A})
\end{cases}
\end{equation*}

and minimize the empirical risk
\begin{equation}\label{eq:loss_hybrid}
\mathcal{L}(\theta) = \frac{1}{N} \sum_{i=1}^N \left( \tilde{g}(\mathbf{A}_i; \theta) - \tilde{\nu}_i \right)^2,
\end{equation}
where $\theta$ denotes the model parameters. At inference time, the condition number estimate is obtained by exponentiating and combining with the exact norm:
\begin{equation}
\label{eq:inference_hybrid}
\hat{\kappa}(\mathbf{A}) =
\begin{cases}
\|\mathbf{A}\| \cdot 10^{\tilde{g}(\mathbf{A}; \theta^*)}, & \text{if norm scaling is used}, \\[4pt]
10^{\tilde{g}(\mathbf{A}; \theta^*)}, & \text{otherwise}.
\end{cases}
\end{equation}
where $\theta^*$ are the optimized parameters. In the following, we refer to the two formulas in \eqref{eq:inference_hybrid} as the prediction schemes, namely, Scheme 1 and Scheme 2. Scheme 1 is used when the norm of the matrix is provided,  and can be used to estimate the norm of inverse as well. Both schemes can be applied to Demmel condition number estimation \cite{Demmel1987ConditionNumbers}.

\subsection{Matrix Feature Modeling}
In this section, we discuss the details of obtaining our feature modeling for GNNs. The direct application of deep learning to matrix inputs faces two fundamental challenges: (i) variable matrix dimensions exclude fixed-size network architectures, and (ii) the computational cost of processing $\mathcal{O}(n^2)$ or $\mathcal{O}(\text{nnz}(\mathbf{A}))$ entries exceeds the cost of traditional methods. We circumvent these obstacles through a carefully designed feature extraction operator $\Phi: \mathbb{R}^{n \times n} \to \mathbb{R}^d$ that maps matrices to fixed-dimensional feature vectors while preserving essential conditioning information.

We associate each sparse matrix $\mathbf{A} \in \mathbb{R}^{n \times n}$ with a compact $d$-dimensional feature vector $\bm{\phi}(\mathbf{A})$ that encodes a rich collection of structural, numerical, and spectral properties known to be predictive of matrix conditioning and the performance of iterative solvers. We write the mapping $\bm{\phi} : \mathbb{R}^{n \times n} \to \mathbb{R}^{d}, \quad
\mathbf{A} \mapsto \bm{\phi}(\mathbf{A})$.  

Following this, we propose a linear complexity algorithm with respect to $\mathrm{nnz}$ for feature extraction. 
Rather than using all available descriptors, we construct a reduced yet informative representation by selecting a subset of components that capture the dominant scale, extremal behavior, and sparsity characteristics of $\mathbf{A}$. 
This design is motivated by the observation that extremal statistics often provide strong proxies for conditioning and worst-case solver behavior, while avoiding redundancy introduced by moment-based descriptors.

The resulting feature vector is formed by concatenating selected components from complementary groups of descriptors, all computable in $\mathcal{O}(\mathrm{nnz}(\mathbf{A}) + n)$ time:
\begin{equation}\label{eq:phi_concat}
\bm{\phi}(\mathbf{A}) 
= \bm{\phi}_{\mathrm{norm}}(\mathbf{A}) \Vert 
\bm{\phi}_{\mathrm{struc}}(\mathbf{A}) \Vert 
\bm{\phi}_{\mathrm{diag}}(\mathbf{A}) \Vert 
\bm{\phi}_{\mathrm{rsp}}(\mathbf{A}) \Vert 
\bm{\phi}_{\mathrm{nzval}}(\mathbf{A}) \Vert 
\bm{\phi}_{\mathrm{gers}}(\mathbf{A}),
\end{equation}
where $\Vert$ denotes vector concatenation, and each term $\bm{\phi}_{\mathrm{norm}}$, $\bm{\phi}_{\mathrm{struc}}$, $\bm{\phi}_{\mathrm{diag}}$, $\bm{\phi}_{\mathrm{rsp}}$, $\bm{\phi}_{\mathrm{nzval}}$, and $\bm{\phi}_{\mathrm{gers}}$ denotes a group of matrix-level descriptors. These groups capture norm-related scaling, structural information, diagonal properties, row-sparsity pattern, nonzero-value statistics, and Gershgorin-based estimates, respectively. The feature vector $\bm{\phi}(\mathbf{A})$ is formed by concatenating these groups in a fixed order, yielding $\bm{\phi}(\mathbf{A}) \in \mathbb{R}^{11}$. The detail of the matrix features used in this paper follows \tablename~\ref{tab:matrix_features}. Here, $\|\mathbf{A}\|_1 = \max_j \sum_i |a_{ij}|$,
$\|\mathbf{A}\|_\infty = \max_i \sum_j |a_{ij}|$, and
$\|\mathbf{A}\|_F = \sqrt{\sum_{i,j} a_{ij}^2}$.
Moreover, $\mathbf{d} = \mathrm{diag}(\mathbf{A})$,
$\rho_i = \mathrm{nnz}(\mathbf{A}_{i,:})$,
$\mathcal{S} = \{a_{ij} : a_{ij} \neq 0\}$, and
$g_i = \sum_{j \neq i} |a_{ij}|$.
The constant $\epsilon$ is a small number (e.g., $10^{-10}$) that protects against underflow.

\begin{table}[htbp]
\centering
\caption{Matrix feature descriptors used in this work.}
\label{tab:matrix_features}
\renewcommand{\arraystretch}{1.25}
\resizebox{0.9\textwidth}{!}{%
\begin{tabular}{p{0.2\linewidth} p{0.78\linewidth}}
\toprule
\textbf{Category} & \textbf{Feature vector and description} \\
\midrule

\textbf{Matrix norms}
&
Scaled $1$-, $\infty$-, and Frobenius norms together with their ratio quantify global magnitude and row/column imbalance, and provide coarse indicators of matrix scaling:
\begin{equation}\label{eq:feat_norm}
\begin{aligned}
\bm{\phi}_{\mathrm{norm}} &= \bigl(
\log_{10}(\|\mathbf{A}\|_1 + \epsilon),\ 
\log_{10}(\|\mathbf{A}\|_\infty + \epsilon),\ 
\log_{10}(\|\mathbf{A}\|_F + \epsilon),\notag \\
&\quad \log_{10}\bigl(\tfrac{\|\mathbf{A}\|_1}{\|\mathbf{A}\|_\infty + \epsilon}\bigr)
\bigr)^\top \in \mathbb{R}^4,
\end{aligned}
\end{equation}
with the standard definitions $\|\mathbf{A}\|_1 = \max_j \sum_i |a_{ij}|$, $\|\mathbf{A}\|_\infty = \max_i \sum_j |a_{ij}|$, and $\|\mathbf{A}\|_F = \sqrt{\sum_{i,j} a_{ij}^2}$.
\\

\textbf{Structural descriptors}
&
We retain the problem size, which serves as a proxy for computational complexity and scaling effects:
\[
\bm{\phi}_{\mathrm{struc}} = \bigl(
\log_{10}(n + 1)
\bigr)^\top \in \mathbb{R}.
\]
\\

\textbf{Diagonal properties}
&
Let $\mathbf{d} = \mathrm{diag}(\mathbf{A})$. We retain range information (in absolute value), which reflects diagonal scaling and potential ill-conditioning:
\[
\begin{aligned}
\bm{\phi}_{\mathrm{diag}} &= \bigl(
\log_{10}(\min |\mathbf{d}| + \epsilon),\ 
\log_{10}(\max |\mathbf{d}| + \epsilon), \\
&\quad \log_{10}\bigl(\tfrac{\max |\mathbf{d}|}{\min |\mathbf{d}| + \epsilon}\bigr)
\bigr)^\top \in \mathbb{R}^3,
\end{aligned}
\]
where $\epsilon = 10^{-10}$ protects against underflow.
\\

\textbf{Row-sparsity pattern descriptors}
&
Let $\rho_i = \mathrm{nnz}(\mathbf{A}_{i,:})$. We retain the maximal row sparsity, capturing worst-case local density:
\[
\begin{aligned}
\bm{\phi}_{\mathrm{rsp}} &= \bigl(
\log_{10}(\max \bm{\rho} + 1)
\bigr)^\top \in \mathbb{R}.
\end{aligned}
\]
\\

\textbf{Nonzero-value statistics}
&
Let $\mathcal{S} = \{a_{ij} : a_{ij} \neq 0\}$. We retain the maximal magnitude to capture extreme numerical values:
\[
\begin{aligned}
\bm{\phi}_{\mathrm{nzval}} &= \bigl(
\log_{10}(\max |\mathcal{S}| + \epsilon)
\bigr)^\top \in \mathbb{R}.
\end{aligned}
\]
\\

\textbf{Gershgorin-based estimates}
&
Define the Gershgorin radii $g_i = \sum_{j \neq i} |a_{ij}|$. We retain the maximum radius, which provides a coarse upper bound on spectral spread:
\[
\bm{\phi}_{\mathrm{gers}} = \bigl(
\log_{10}(\max \mathbf{g} + \epsilon)
\bigr)^\top \in \mathbb{R}.
\]
\\
\bottomrule
\end{tabular}
}
\end{table}

The assembly of $\bm{\phi}(\mathbf{A})$ as in \eqref{eq:phi_concat} requires only arithmetic operations $\mathcal{O}(\mathrm{nnz}(\mathbf{A}) + n)$, making feature extraction negligible compared to any subsequent condition-number estimation procedure. 
These statistics provide a fixed-dimensional input representation for the GNN and enable an end-to-end learning framework for condition-number prediction. To explain their analytical relevance, we examine the bounds they provide for the $1$-norm and $2$-norm condition numbers under the structural assumptions stated below. The analysis uses the exact, unlogged quantities underlying the features; the stabilizing constant $\epsilon$ is used only in feature computation and is omitted from the analytical bounds.

For reference, the analysis uses the following 11-dimensional subvector,
denoted by $\widetilde{\bm{\phi}}(\mathbf A)\in\mathbb R^{11}$:
\begin{equation}\label{eq:feature1}
\begin{aligned}
\widetilde{\bm{\phi}}(\mathbf A)=\Bigl(
&\log_{10}(\|\mathbf A\|_1+\epsilon),\ 
\log_{10}(\|\mathbf A\|_\infty+\epsilon),\ 
\log_{10}(\|\mathbf A\|_F+\epsilon), \\
&\log_{10}\Bigl(\frac{\|\mathbf A\|_1+\epsilon}
{\|\mathbf A\|_\infty+\epsilon}\Bigr),\ 
\log_{10}(\min |\mathbf d|+\epsilon), \\
&\log_{10}(\max \mathbf g+\epsilon),\ 
\log_{10}(n+1),\ 
\log_{10}(\max |\mathbf d|+\epsilon), \\
&\log_{10}\Bigl(\frac{\max |\mathbf d|+\epsilon}
{\min |\mathbf d|+\epsilon}\Bigr),\ 
\log_{10}(\max \bm{\rho}+1), \\
&\log_{10}(\max |\mathcal S|+\epsilon)
\Bigr)^\top,
\end{aligned}
\end{equation}
where $\mathbf d=\mathrm{diag}(\mathbf A)$, 
$\rho_i=\mathrm{nnz}(\mathbf A_{i,:})$, 
$\bm{\rho}=(\rho_1,\ldots,\rho_n)^\top$, 
$\mathcal S=\{a_{ij}:a_{ij}\neq 0\}$, and 
$g_i=\sum_{j\neq i}|a_{ij}|$.

We first examine the norm features together with the matrix dimension. The following result shows how these quantities bound the spectral norm and provide a lower bound on the $1$-norm condition number.

\begin{theorem}
\label{thm:norm_feature_bounds}
Let $\mathbf{A}\in\mathbb{R}^{n\times n}$, and define
\begin{equation*}
\alpha:=\|\mathbf{A}\|_1,\qquad
\beta:=\|\mathbf{A}\|_\infty,\qquad
\gamma:=\|\mathbf{A}\|_F,\qquad
r:=\frac{\alpha}{\beta}.
\end{equation*}
Then
\begin{equation}
\label{eq:two_norm_interval}
\max\!\left\{\frac{\alpha}{\sqrt{n}},\frac{\beta}{\sqrt{n}},\frac{\gamma}{\sqrt{n}}\right\}
\le
\|\mathbf{A}\|_2
\le
\min\!\left\{\gamma,\sqrt{\alpha\beta}\right\}
=
\min\!\left\{\gamma,\frac{\alpha}{\sqrt{r}}\right\}.
\end{equation}
In particular, the feature group $\bm{\phi}_{\mathrm{norm}}(\mathbf{A})$ determines rigorous lower and upper bounds for $\|\mathbf{A}\|_2$.

If, in addition, $\mathbf{A}$ is nonsingular, then
\begin{equation}
\label{eq:kappa1_lower_from_norm_features}
\kappa_1(\mathbf{A})
\ge
\max\!\left\{
1,\,
\frac{\alpha}{\min\{\gamma,\sqrt{\alpha\beta}\}}
\right\}
=
\max\!\left\{
1,\,
\frac{\alpha}{\gamma},\,
\sqrt{r}
\right\}.
\end{equation}
\end{theorem}

\begin{proof}
Following the inequalities
\begin{equation*}
\|\mathbf{A}\|_1\le \sqrt{n}\,\|\mathbf{A}\|_2,
\qquad
\|\mathbf{A}\|_\infty\le \sqrt{n}\,\|\mathbf{A}\|_2,
\end{equation*}

We have
\begin{equation*}
\frac{\alpha}{\sqrt{n}}\le \|\mathbf{A}\|_2,
\qquad
\frac{\beta}{\sqrt{n}}\le \|\mathbf{A}\|_2.
\end{equation*}
Also, if $\sigma_1(\mathbf{A}),\dots,\sigma_n(\mathbf{A})$ denote the singular values of $\mathbf{A}$, then
\begin{equation*}
\gamma^2
=
\|\mathbf{A}\|_F^2
=
\sum_{i=1}^n \sigma_i(\mathbf{A})^2
\le
n\,\sigma_{\max}(\mathbf{A})^2
=
n\,\|\mathbf{A}\|_2^2,
\end{equation*}
whence
\begin{equation*}
\frac{\gamma}{\sqrt{n}}\le \|\mathbf{A}\|_2.
\end{equation*}
This proves the left-hand side of \eqref{eq:two_norm_interval}.

For the upper bounds, the standard inequalities
\begin{equation*}
\|\mathbf{A}\|_2\le \|\mathbf{A}\|_F
\qquad\text{and}\qquad
\|\mathbf{A}\|_2\le \sqrt{\|\mathbf{A}\|_1\|\mathbf{A}\|_\infty}
\end{equation*}
yield
\begin{equation*}
\|\mathbf{A}\|_2\le \gamma,
\qquad
\|\mathbf{A}\|_2\le \sqrt{\alpha\beta}.
\end{equation*}
Therefore
\begin{equation*}
\|\mathbf{A}\|_2\le \min\!\left\{\gamma,\sqrt{\alpha\beta}\right\}.
\end{equation*}
Since $r=\alpha/\beta$, we also have
\begin{equation*}
\sqrt{\alpha\beta}=\frac{\alpha}{\sqrt{r}},
\end{equation*}
which completes the proof of \eqref{eq:two_norm_interval}.

Now assume that $\mathbf{A}$ is nonsingular. Then
\begin{equation*}
\|\mathbf{A}^{-1}\|_1\ge \|\mathbf{A}^{-1}\|_2
=
\frac{1}{\sigma_{\min}(\mathbf{A})}
\ge
\frac{1}{\sigma_{\max}(\mathbf{A})}
=
\frac{1}{\|\mathbf{A}\|_2}.
\end{equation*}
Hence
\begin{equation*}
\kappa_1(\mathbf{A})
=
\|\mathbf{A}\|_1\|\mathbf{A}^{-1}\|_1
\ge
\frac{\alpha}{\|\mathbf{A}\|_2}.
\end{equation*}

Substituting the upper bound for $\|\mathbf{A}\|_2$ from \eqref{eq:two_norm_interval} gives
\begin{equation*}
\kappa_1(\mathbf{A})
\ge
\frac{\alpha}{\min\{\gamma,\sqrt{\alpha\beta}\}}
=
\max\!\left\{\frac{\alpha}{\gamma},\sqrt{\frac{\alpha}{\beta}}\right\}
=
\max\!\left\{\frac{\alpha}{\gamma},\sqrt{r}\right\}.
\end{equation*}
Since every matrix condition number satisfies $\kappa_1(\mathbf A)\ge 1$, we obtain the strengthened bound
\begin{equation*}
\kappa_1(\mathbf A)
\ge
\max\!\left\{
1,\,
\frac{\alpha}{\gamma},\,
\sqrt r
\right\}.
\end{equation*}
This proves \eqref{eq:kappa1_lower_from_norm_features}.
\end{proof}

\begin{remark}
\label{rem:no_kappa2_upper_from_norm_only}
The feature group $\bm{\phi}_{\mathrm{norm}}(\mathbf{A})$ does not determine $\sigma_{\min}(\mathbf{A})$. Consequently, it does not by itself provide a nontrivial upper bound for either $\kappa_2(\mathbf{A})$ or $\kappa_1(\mathbf{A})$.
\end{remark}

To obtain an upper bound on the condition number, additional information is needed to control the smallest singular value. For symmetric matrices, the diagonal entries and Gershgorin radii provide such control when their separation is positive, as formalized in the following result.

\begin{theorem}
\label{thm:gers_bounds_correct}
Let $\mathbf{A}=\mathbf{A}^\top\in\mathbb{R}^{n\times n}$ be symmetric, and define
\begin{equation*}
g_i:=\sum_{j\neq i}|a_{ij}|,\qquad
d_{\min}:=\min_i a_{ii},\qquad
d_{\max}:=\max_i a_{ii},\qquad
R_{\max}:=\max_i g_i.
\end{equation*}
Assume that $d_{\min}>R_{\max}$, in particular, this assumption implies $d_{\min}>0$, so the diagonal entries are positive and the diagonal quantities agree with the absolute-value diagonal features used in $\widetilde{\bm{\phi}}(\mathbf A)$.
Then $\mathbf{A}$ is symmetric positive definite and
\begin{equation}
\label{eq:eig_interval_gers}
d_{\min}-R_{\max}
\le
\lambda_{\min}(\mathbf{A})
\le
\lambda_{\max}(\mathbf{A})
\le
d_{\max}+R_{\max}.
\end{equation}
Consequently,
\begin{equation}
\label{eq:kappa2_gers_exact}
\kappa_2(\mathbf{A})
=
\frac{\lambda_{\max}(\mathbf{A})}{\lambda_{\min}(\mathbf{A})}
\le
\frac{d_{\max}+R_{\max}}{d_{\min}-R_{\max}}.
\end{equation}
Moreover,
\begin{equation}
\label{eq:kappa2_gers_with_norms}
\kappa_2(\mathbf{A})
\le
\frac{\min\{\|\mathbf{A}\|_F,\sqrt{\|\mathbf{A}\|_1\|\mathbf{A}\|_\infty},\,d_{\max}+R_{\max}\}}
{d_{\min}-R_{\max}}.
\end{equation}
\end{theorem}

\begin{proof}
By Gershgorin's theorem, every eigenvalue $\lambda$ of $\mathbf{A}$ belongs to at least one interval
\begin{equation*}
[a_{ii}-g_i,\;a_{ii}+g_i].
\end{equation*}
Therefore,
\begin{equation*}
\lambda\ge \min_i(a_{ii}-g_i)\ge d_{\min}-R_{\max},
\end{equation*}
and similarly,
\begin{equation*}
\lambda\le \max_i(a_{ii}+g_i)\le d_{\max}+R_{\max}.
\end{equation*}
This proves \eqref{eq:eig_interval_gers}. Since $d_{\min}>R_{\max}$, all eigenvalues are strictly positive, and hence $\mathbf{A}$ is symmetric positive definite.

For a symmetric positive definite matrix,
\begin{equation*}
\kappa_2(\mathbf{A})
=
\frac{\lambda_{\max}(\mathbf{A})}{\lambda_{\min}(\mathbf{A})},
\end{equation*}
so \eqref{eq:kappa2_gers_exact} follows immediately from \eqref{eq:eig_interval_gers}.

Finally, for any matrix,
\begin{equation*}
\|\mathbf{A}\|_2\le \min\{\|\mathbf{A}\|_F,\sqrt{\|\mathbf{A}\|_1\|\mathbf{A}\|_\infty}\},
\end{equation*}
whereas in the present symmetric positive definite setting,
\begin{equation*}
\lambda_{\max}(\mathbf{A})=\|\mathbf{A}\|_2\le d_{\max}+R_{\max}.
\end{equation*}
Combining these upper bounds for $\lambda_{\max}(\mathbf{A})$ with
\begin{equation*}
\lambda_{\min}(\mathbf{A})\ge d_{\min}-R_{\max}
\end{equation*}
yields \eqref{eq:kappa2_gers_with_norms}.
\end{proof}

\begin{remark}
\label{rem:gers_nonsym}
For a general nonsymmetric matrix, Gershgorin's theorem bounds eigenvalues but not singular values. Accordingly, the implication from $|\lambda|\ge d_{\min}-R_{\max}$ to $\sigma_{\min}(\mathbf{A})\ge d_{\min}-R_{\max}$ 
is false in general. A Gershgorin-based upper bound for $\kappa_2(\mathbf{A})$ therefore requires additional structure, such as symmetry.
\end{remark}

The preceding result gives an upper bound under a positive diagonal-dominance margin. A complementary lower bound follows from the diagonal range alone and holds for every symmetric positive definite matrix.

\begin{theorem}
\label{thm:diag_range_lower}
Let $\mathbf{A}=\mathbf{A}^\top\in\mathbb{R}^{n\times n}$ be symmetric positive definite, and define
\begin{equation*}
d_{\min}:=\min_{1\le i\le n} a_{ii},
\qquad
d_{\max}:=\max_{1\le i\le n} a_{ii}.
\end{equation*}
Since $\mathbf A$ is positive definite, all diagonal entries are positive. Hence these quantities coincide with the absolute diagonal quantities encoded by the diagonal features, namely
\[
d_{\min}=\min_i |a_{ii}|,
\qquad
d_{\max}=\max_i |a_{ii}|.
\]
Then
\begin{equation}
\label{eq:diag_range_kappa2_lower}
\kappa_2(\mathbf{A})\ge \frac{d_{\max}}{d_{\min}}.
\end{equation}
Consequently,
\begin{equation}
\label{eq:diag_range_kappa1_lower}
\kappa_1(\mathbf{A})\ge \kappa_2(\mathbf{A})\ge \frac{d_{\max}}{d_{\min}}.
\end{equation}

Hence the diagonal-range feature, whose unlogged counterpart is $\frac{d_{\max}}{d_{\min}}$, provides a rigorous lower bound for both $\kappa_2(\mathbf{A})$ and $\kappa_1(\mathbf{A})$, up to the numerical stabilizer $\epsilon$ used in feature computation.
\end{theorem}

\begin{proof}
Since $\mathbf{A}$ is symmetric positive definite,
\begin{equation*}
\kappa_2(\mathbf{A})=\frac{\lambda_{\max}(\mathbf{A})}{\lambda_{\min}(\mathbf{A})}.
\end{equation*}
Let $e_i$ denote the $i$th coordinate vector. By the Rayleigh quotient characterization,
\begin{equation*}
\lambda_{\max}(\mathbf{A})\ge e_k^\top \mathbf{A} e_k = a_{kk}=d_{\max}
\end{equation*}
for an index $k$ attaining the maximum diagonal entry, and similarly
\begin{equation*}
\lambda_{\min}(\mathbf{A})\le e_\ell^\top \mathbf{A} e_\ell = a_{\ell\ell}=d_{\min}
\end{equation*}
for an index $\ell$ attaining the minimum diagonal entry. Therefore
\begin{equation*}
\kappa_2(\mathbf{A})
=
\frac{\lambda_{\max}(\mathbf{A})}{\lambda_{\min}(\mathbf{A})}
\ge
\frac{d_{\max}}{d_{\min}},
\end{equation*}
which proves \eqref{eq:diag_range_kappa2_lower}.

Finally, since $\mathbf A$ is symmetric positive definite, both $\mathbf A$ and $\mathbf A^{-1}$ are symmetric. Therefore,
\[
\|\mathbf A\|_1=\|\mathbf A\|_\infty,
\qquad
\|\mathbf A^{-1}\|_1=\|\mathbf A^{-1}\|_\infty.
\]
Using the standard inequality
\[
\|\mathbf B\|_2\le \sqrt{\|\mathbf B\|_1\|\mathbf B\|_\infty},
\]
we obtain
\[
\|\mathbf A\|_2\le \|\mathbf A\|_1,
\qquad
\|\mathbf A^{-1}\|_2\le \|\mathbf A^{-1}\|_1.
\]
Hence
\begin{equation*}
\kappa_2(\mathbf{A})
=
\|\mathbf{A}\|_2\|\mathbf{A}^{-1}\|_2
\le
\|\mathbf{A}\|_1\|\mathbf{A}^{-1}\|_1
=
\kappa_1(\mathbf{A}),
\end{equation*}
and \eqref{eq:diag_range_kappa1_lower} follows.
\end{proof}

The preceding bounds relate conditioning to diagonal variation and off-diagonal coupling. To connect this analysis with the row-sparsity and nonzero-value features, we next bound the Gershgorin radii using the maximum number of nonzeros per row and the largest entry magnitude. The following result combines this estimate with the diagonal-range lower bound under the stated separation condition.

\begin{theorem}
\label{thm:rsp_nzval_bounds}
Let $\mathbf{A}=\mathbf{A}^\top\in\mathbb{R}^{n\times n}$ be symmetric with positive diagonal, and define
\begin{equation*}
d_{\min}:=\min_i a_{ii},
\qquad
d_{\max}:=\max_i a_{ii},
\qquad
\rho_{\max}:=\max_i \mathrm{nnz}(\mathbf{A}_{i,:}),
\qquad
v_{\max}:=\max_{a_{ij}\neq 0}|a_{ij}|.
\end{equation*}
Assume that
\begin{equation}
\label{eq:sparsity_value_dominance}
d_{\min}>(\rho_{\max}-1)v_{\max}.
\end{equation}
Then $\mathbf{A}$ is symmetric positive definite and
\begin{equation}
\label{eq:kappa2_rsp_nzval_interval}
\frac{d_{\max}}{d_{\min}}
\le
\kappa_2(\mathbf{A})
\le
\frac{d_{\max}+(\rho_{\max}-1)v_{\max}}
{d_{\min}-(\rho_{\max}-1)v_{\max}}.
\end{equation}
Moreover,
\begin{equation}
\label{eq:kappa1_rsp_nzval_interval}
\frac{d_{\max}}{d_{\min}}
\le
\kappa_1(\mathbf{A})
\le
\frac{d_{\max}+(\rho_{\max}-1)v_{\max}}
{d_{\min}-(\rho_{\max}-1)v_{\max}}.
\end{equation}
All quantities in \eqref{eq:kappa2_rsp_nzval_interval}--\eqref{eq:kappa1_rsp_nzval_interval} are the unlogged counterparts of the corresponding components in $\bm{\phi}_{\mathrm{diag}}(\mathbf{A})$, $\bm{\phi}_{\mathrm{rsp}}(\mathbf{A})$, and $\bm{\phi}_{\mathrm{nzval}}(\mathbf{A})$, up to the numerical stabilizer $\epsilon$ used in feature computation.
\end{theorem}
\begin{proof}
For each row $i$, let
\begin{equation*}
g_i:=\sum_{j\neq i}|a_{ij}|.
\end{equation*}
Although $v_{\max}$ is taken over all nonzero entries, including diagonal entries, it is still a valid upper bound for the magnitude of every off-diagonal nonzero entry. Since row $i$ contains at most $\rho_i-1$ off-diagonal nonzeros, each of magnitude at most $v_{\max}$, we have
\begin{equation*}
g_i\le (\rho_i-1)v_{\max}\le (\rho_{\max}-1)v_{\max}.
\end{equation*}
Hence
\begin{equation*}
a_{ii}-g_i
\ge
d_{\min}-(\rho_{\max}-1)v_{\max}
>0
\end{equation*}
by \eqref{eq:sparsity_value_dominance}. Therefore $\mathbf{A}$ is strictly diagonally dominant by rows. Because $\mathbf{A}$ is symmetric, Gershgorin's theorem implies that every eigenvalue $\lambda$ of $\mathbf{A}$ lies in an interval
\begin{equation*}
[a_{ii}-g_i,\; a_{ii}+g_i]
\subseteq
\bigl[d_{\min}-(\rho_{\max}-1)v_{\max},\;
d_{\max}+(\rho_{\max}-1)v_{\max}\bigr].
\end{equation*}
It follows that
\begin{equation*}
0<d_{\min}-(\rho_{\max}-1)v_{\max}
\le \lambda_{\min}(\mathbf{A})
\le \lambda_{\max}(\mathbf{A})
\le d_{\max}+(\rho_{\max}-1)v_{\max},
\end{equation*}
so $\mathbf{A}$ is symmetric positive definite. Consequently,
\begin{equation*}
\kappa_2(\mathbf{A})
=
\frac{\lambda_{\max}(\mathbf{A})}{\lambda_{\min}(\mathbf{A})}
\le
\frac{d_{\max}+(\rho_{\max}-1)v_{\max}}
{d_{\min}-(\rho_{\max}-1)v_{\max}}.
\end{equation*}
The lower bound in \eqref{eq:kappa2_rsp_nzval_interval} follows from Theorem~\ref{thm:diag_range_lower}. This proves \eqref{eq:kappa2_rsp_nzval_interval}.

Since $\mathbf{A}$ is symmetric, $\mathbf{A}^{-1}$ is also symmetric, and therefore
\begin{equation*}
\|\mathbf{A}\|_1=\|\mathbf{A}\|_\infty,
\qquad
\|\mathbf{A}^{-1}\|_1=\|\mathbf{A}^{-1}\|_\infty.
\end{equation*}
Moreover,
\begin{equation*}
\|\mathbf{A}\|_\infty
=
\max_i\sum_j |a_{ij}|
\le
d_{\max}+(\rho_{\max}-1)v_{\max}.
\end{equation*}
Also,
\begin{equation*}
\min_i(a_{ii}-g_i)\ge d_{\min}-(\rho_{\max}-1)v_{\max}>0.
\end{equation*}
Hence, by Varah's bound for strictly row diagonally dominant matrices,
\begin{equation*}
\|\mathbf{A}^{-1}\|_\infty
\le
\frac{1}{d_{\min}-(\rho_{\max}-1)v_{\max}}.
\end{equation*}
Therefore
\begin{equation*}
\kappa_1(\mathbf{A})
=
\|\mathbf{A}\|_1\|\mathbf{A}^{-1}\|_1
=
\|\mathbf{A}\|_\infty\|\mathbf{A}^{-1}\|_\infty
\le
\frac{d_{\max}+(\rho_{\max}-1)v_{\max}}
{d_{\min}-(\rho_{\max}-1)v_{\max}}.
\end{equation*}
Finally, the lower bound in \eqref{eq:kappa1_rsp_nzval_interval} follows from Theorem~\ref{thm:diag_range_lower}. This completes the proof.
\end{proof}

The dominance margin is critical for condition-number control. The following remark gives a simple example illustrating why such a margin is necessary.
\begin{remark}[Why a dominance margin is necessary]
\label{rem:need_margin}
The separation condition in Theorem~\ref{thm:rsp_nzval_bounds} is essential. Indeed, consider
\begin{equation*}
\mathbf{A}_\varepsilon
=
\begin{bmatrix}
1 & 1-\varepsilon\\
1-\varepsilon & 1
\end{bmatrix},
\qquad 0<\varepsilon<1.
\end{equation*}
Then
\begin{equation*}
d_{\min}=d_{\max}=1,\qquad
\rho_{\max}=2,\qquad
v_{\max}=1,
\end{equation*}
all of which remain bounded independently of $\varepsilon$, whereas
\begin{equation*}
\kappa_2(\mathbf{A}_\varepsilon)=\frac{2-\varepsilon}{\varepsilon}\to\infty
\qquad\text{as }\varepsilon\rightarrow 0.
\end{equation*}
Thus diagonal range, row sparsity, and maximum entry magnitude do not by themselves control the condition number unless one imposes a positive dominance margin such as \eqref{eq:sparsity_value_dominance}.
\end{remark}

The results above show that the selected 11-dimensional subvector
$\widetilde{\bm{\phi}}(\mathbf A)$ has explicit analytical relevance to matrix conditioning: it provides norm bounds, spectral bounds under symmetry, and condition-number intervals under diagonal-dominance-type assumptions.

\begin{remark}[Scope of the analytical bounds]
The bounds derived above do not imply that the selected features fully determine the condition number of an arbitrary matrix. In particular, without additional structural assumptions, such as symmetry, positive definiteness, or a positive diagonal-dominance margin, the minimum singular value may be arbitrarily small while the selected extremal features remain bounded. Therefore, the theoretical results should be interpreted as partial analytical justification for the feature design rather than as a complete characterization of the condition number.
\end{remark}



\subsection{Graph Neural Network Architecture}

To leverage the structural information encoded in the sparsity pattern of $\mathbf{A}$, we represent each matrix as an attributed graph
\[
\mathcal{G} = (\mathcal{V}, \mathcal{E}, \mathbf{X}, \bm{\phi}),
\]
where:
\begin{itemize}
\item $\mathcal{V} = \{1, \ldots, n\}$ is the set of nodes corresponding to matrix rows/columns;
\item $\mathcal{E} = \{(i,j) : a_{ij} \neq 0\}$ is the directed edge set induced by the nonzero pattern of $\mathbf{A}$;
\item $\mathbf{X} \in \mathbb{R}^{n \times 2}$ is the node feature matrix with rows
\begin{equation}\label{eq:node_features}
\mathbf{x}_i =
\bigl(
\log_{10}(|a_{ii}| + \epsilon),\
\log_{10}(\rho_i + 1)
\bigr)^\top,
\end{equation}
where $\rho_i = \mathrm{nnz}(\mathbf{A}_{i,:})$ denotes the row sparsity;
\item $\bm{\phi}(\mathbf{A}) \in \mathbb{R}^{d}$ is the global feature vector defined in \eqref{eq:phi_concat}.
\end{itemize}

To instantiate the graph neural network (GNN) component, we employ a $K$-layer graph convolutional network (GCN) \cite{kipf2017semi}, a representative message-passing GNN architecture, to learn node embeddings from the unweighted graph induced by the sparsity pattern of $\mathbf{A}$. To be specific, the graph convolutional layers operate on the unweighted connectivity pattern specified by $\mathcal{E}$; that is, the message-passing graph uses the sparsity structure of $\mathbf{A}$ but does not assign additional learned or matrix-value-dependent edge attributes. One may seek improved performance via including edge attributes $e_{ij}=\log_{10}(|a_{ij}|+\epsilon)$ but in our GCN implementation these attributes are not directly used by the message-passing layers. The GCN branch uses the sparsity pattern through edge index while numerical magnitude information is incorporated through node features and graph-level statistical features. Numerical magnitude information is incorporated through the node features in \eqref{eq:node_features} and the global matrix descriptor $\bm{\phi}(\mathbf{A})$.

Additionally, when required by Scheme~1 in the inference formula \eqref{eq:inference_hybrid}, we store the exact matrix norm $\|\mathbf{A}\|_p$ as graph-level metadata. In particular, this corresponds to $\|\mathbf{A}\|_1$ for the 1-norm condition number and $\|\mathbf{A}\|_2$ for the 2-norm condition number.

\subsubsection{Message Passing Layers}
 The GCN branch uses the edge set $\mathcal{E}$ induced by the nonzero pattern, while matrix-value information is incorporated through the node features in \eqref{eq:node_features} and the global descriptor $\bm{\phi}(\mathbf{A})$.

The initial node embedding is obtained by a linear encoder followed by a nonlinear activation:
\begin{equation}
\mathbf{h}_i^{(0)}
=
\sigma\left(
\mathbf{W}_{\mathrm{node}}\mathbf{x}_i
+
\mathbf{b}_{\mathrm{node}}
\right),
\label{eq:node_encoder}
\end{equation}
where $\mathbf{W}_{\mathrm{node}}\in\mathbb{R}^{\ell\times 2}$, $\mathbf{b}_{\mathrm{node}}\in\mathbb{R}^{\ell}$, $\ell$ is the hidden dimension, and $\sigma(\cdot)$ denotes the ReLU activation.

Following the standard GCN normalization, we augment the graph with self-loops:
\[
\tilde{\mathcal{E}}
=
\mathcal{E}
\cup
\{(i,i):i\in\mathcal{V}\}.
\]
Let $\tilde{\mathcal{N}}^{-}(i) = \{j:(j,i)\in\tilde{\mathcal{E}}\}$
be the set of source nodes that send messages to node $i$, and let
$\tilde d_i = |\tilde{\mathcal{N}}^{-}(i)|$ be the corresponding degree. The $k$-th GCN layer is then written as
\begin{equation}
\mathbf{h}_i^{(k+1)}
=
\sigma\left(
\sum_{j\in \tilde{\mathcal{N}}^{-}(i)}
\frac{1}{\sqrt{\tilde d_i\tilde d_j}}\,
\mathbf{W}^{(k)}\mathbf{h}_j^{(k)}
+
\mathbf{b}^{(k)}
\right),
\qquad
k=0,\ldots,K-1,
\label{eq:gcn_propagation}
\end{equation}
where $\mathbf{W}^{(k)}\in\mathbb{R}^{\ell\times\ell}$ and $\mathbf{b}^{(k)}\in\mathbb{R}^{\ell}$ are learnable parameters. This propagation rule corresponds to message passing over the unweighted graph induced by the sparsity pattern of $\mathbf{A}$.

\subsubsection{Graph-Level Aggregation and Prediction}

The final node embeddings $\{\mathbf{h}_i^{(K)}\}_{i=1}^n$ are mapped to a graph-level representation using permutation-invariant mean and max readouts:
\begin{equation}
\mathbf{z}_{\mathrm{mean}}
=
\frac{1}{n}
\sum_{i=1}^n
\mathbf{h}_i^{(K)},
\qquad
\mathbf{z}_{\max}
=
\max_{1\leq i\leq n}
\mathbf{h}_i^{(K)},
\label{eq:graph_pooling}
\end{equation}
where the maximum is taken element-wise. In parallel, the global matrix descriptor $\bm{\phi}(\mathbf{A})$ is encoded by a separate fully connected branch:
\begin{equation}
\mathbf{u}
=
\sigma\left(
\mathbf{W}_{\mathrm{global}}\bm{\phi}(\mathbf{A})
+
\mathbf{b}_{\mathrm{global}}
\right),
\label{eq:global_encoder}
\end{equation}
with $\mathbf{W}_{\mathrm{global}}\in\mathbb{R}^{\ell\times d}$ and $\mathbf{b}_{\mathrm{global}}\in\mathbb{R}^{\ell}$.

The graph-level representation is formed by concatenating the local structural readouts and the encoded global descriptor:
\begin{equation}
\mathbf{q}(\mathbf{A})
=
\left[
\mathbf{z}_{\mathrm{mean}};
\mathbf{z}_{\max};
\mathbf{u}
\right]
\in\mathbb{R}^{3\ell}.
\label{eq:combined_representation}
\end{equation}
This representation is processed by a multilayer prediction head:
\begin{equation}
\begin{aligned}
\mathbf{s}^{(0)}
&=
\mathbf{q}(\mathbf{A}),\\
\mathbf{s}^{(m)}
&=
\sigma\left(
\mathrm{Drop}_q
\left(
\mathbf{W}_{\mathrm{out}}^{(m)}
\mathbf{s}^{(m-1)}
+
\mathbf{b}_{\mathrm{out}}^{(m)}
\right)
\right),
\qquad
m=1,\ldots,L-1,\\
\tilde{g}(\mathbf{A};\theta)
&=
\mathbf{W}_{\mathrm{out}}^{(L)}
\mathbf{s}^{(L-1)}
+
\mathbf{b}_{\mathrm{out}}^{(L)}.
\end{aligned}
\label{eq:gnn_predictor_general}
\end{equation}
Here $\mathrm{Drop}_q$ denotes dropout with rate $q$ \cite{srivastava2014dropout}.

For the prediction scheme 1, the scalar output $\tilde{g}(\mathbf{A};\theta)$ is then used in the corresponding inference formula:
\begin{equation*}
\hat{\kappa}_p(\mathbf{A})
=
\|\mathbf{A}\|_p
\cdot
10^{\tilde{g}(\mathbf{A};\theta^*)},
\label{eq:final_estimate}
\end{equation*}
where $\|\mathbf{A}\|_p$ is supplied explicitly in \eqref{eq:final_estimate}, thus the neural network is trained to approximate the remaining inverse-norm factor in logarithmic scale.

While for prediction scheme 2, the scalar output follows:
\begin{equation*}
\hat{\kappa}_p(\mathbf{A})
= 
10^{\tilde{g}(\mathbf{A};\theta^*)}.
\end{equation*}

The resulting two-stream architecture combines a local GCN branch, which captures the sparsity connectivity of $\mathbf{A}$, with a global multilayer perceptron branch, which incorporates matrix-level norm, diagonal, magnitude, and sparsity statistics.


\section{Numerical Experiments}\label{sec:exps}
We evaluate the performance of our method against standard approaches for condition number estimation. Unless otherwise noted, all methods are implemented in PyTorch. Experiments and timing measurements were performed on a server equipped with 1 TB of RAM, dual Intel Xeon Gold 6330 processors (56 cores and 112 threads per processor, 2.00 GHz base frequency), and four NVIDIA A100 GPUs (80 GB PCIe each). All computations were executed on the a single GPU.

While condition-number estimation for sparse matrices has well-established CPU-based approaches, to the best of our knowledge, no widely adopted general-purpose GPU routine is available for this task in mainstream sparse linear algebra libraries. In particular, the optimized solution in Python is lacking. The mainstream Python GPU libraries do not currently provide a dedicated sparse-matrix condition-number API. PyTorch exposes \texttt{torch.linalg.cond} for dense tensors as a direct call for calculating condition number, while CuPy provides sparse SVD and Krylov subroutines that can be used to construct condition-number estimators, but not a direct sparse condition-number routine. In PyTorch, as of early 2026, \texttt{torch.linalg.cond} is not documented as a sparse-specialized condition-number estimator. According to its documented behavior, it computes condition numbers using standard dense linear algebra operations: singular-value computations for the $2$-norm-related cases, and matrix norm plus matrix inverse computations for norms such as the Frobenius, nuclear, $1$-, and $\infty$-norms.
Although CuPy\footnote{version 14.0.1, \url{https://cupy.dev/}} provides cuSPARSE-backed sparse matrix operations, our preliminary experiments showed that it was consistently slower than PyTorch for the algorithms considered in this study, across all tested data scales. We therefore restrict the subsequent evaluation to PyTorch in 2.12.1 with CUDA12.6.

For a comprehensive and fair evaluation, we consider several implementations of the competing approaches. For computing the 1-norm condition number, we consider the following methods: $\mathcircled{1}$ computing the exact value using \texttt{torch.linalg.cond} on GPU;
$\mathcircled{2}$ using \texttt{scipy.sparse.linalg.norm} to compute the matrix 1-norm, \texttt{scipy.sparse.linalg.splu} to perform linear solves required for matrix inversion, and \texttt{scipy.sparse.linalg.onenormest}\footnote{A Block Algorithm for Matrix 1-Norm Estimation, note that it is to compute a lower bound of the 1-norm of a sparse array (matrix).} to estimate the 1-norm of the matrix inverse; $\mathcircled{3}$ applying the Hager–Higham algorithm with \texttt{torch.linalg.lu\_factor} for LU solves, while keeping the remaining operations implemented in PyTorch on GPU. The PyTorch-based implementations for 1-norm condition number evaluation are performed on both the CPU and the GPU.  Methods $\mathcircled{2}$ and $\mathcircled{3}$ are referred to as Higham-SciPy and Hager–Higham, respectively, in the following.

For the computation of the 2-norm condition number, we consider two approaches: $\mathcircled{4}$ computing the exact value using \texttt{torch.linalg.cond} on GPU; $\mathcircled{5}$ a GPU-based Golub–Kahan method, i.e., a condition number estimate based on Ritz singular values computed via Golub–Kahan bidiagonalization; $\mathcircled{6}$ SciPy's sparse singular value decomposition (\texttt{scipy.sparse.linalg.svds}), which extracts the largest and smallest singular values of the sparse matrix. Methods $\mathcircled{1}$ and $\mathcircled{4}$ are referred to as the exact approaches.  Our proposed approach is referred to as GNN in the following. Our Golub–Kahan method takes the Golub–Kahan Lanczos bidiagonalization for estimating the 2-norm condition number of a matrix. To enhance numerical stability, full reorthogonalization is applied during the iterations using the Modified Gram–Schmidt procedure. The sparse matrices are stored as PyTorch sparse tensors to enable GPU-accelerated computations.

\subsection{Training Protocols}

We minimize the loss function~\eqref{eq:loss_hybrid} augmented with $L^2$ regularization:
\begin{equation*}\label{eq:reg_loss}
\mathcal{L}_\text{reg}(\theta) = \mathcal{L}(\theta) + \lambda \sum_{L} \|\mathbf{W}^{(L)}\|_F^2,
\end{equation*}
with regularization parameter $\lambda = 10^{-5}$. We use the Adam algorithm optimizer~\cite{kingma2015adam} with learning rate $\alpha = 10^{-3}$ and default momentum parameters $\beta_1 = 0.9$, $\beta_2 = 0.999$. Training proceeds for up to $100$ epochs with mini-batch size $32$, employing early stopping with patience $20$ based on validation set performance.

The learning rate is adaptively reduced by a factor of $0.5$ when validation loss plateaus for $10$ consecutive epochs, implementing the ReduceLROnPlateau schedule~\cite{paszke2019pytorch}. Gradient clipping with threshold $\tau = 1.0$ is applied to prevent exploding gradients during training. Dropout regularization~\cite{srivastava2014dropout} with a rate of $0.1$ is applied after the first two hidden layers of the prediction head to mitigate overfitting prone to model training.

For the GNN-based estimator, we report two timing metrics.
GNN (total) denotes the full pipeline runtime, including conversion of the sparse matrix into the graph representation required by the network, exact computation of $\|A\|_2$, and subsequent model inference.
GNN (inference) denotes only the pure forward-pass time of the trained GNN, excluding preprocessing and norm computation.

\subsection{Evaluation Metrics}
For the reference result on the exact value of the condition number, we use the function \texttt{torch.linalg.cond} to compute the condition number $\kappa_p(A) = \|\mathbf{A}\|_p \cdot \|\mathbf{A}^{-1}\|_p$ depending on the chosen norm. When $p=2$ (i.e., 2-norm), the library employs singular value decomposition to factorize $A$, and defines the 2-norm condition number as the ratio of the largest to the smallest singular value, i.e., $\kappa_2(A) = \sigma_{\max} / \sigma_{\min}$. This approach delivers high computational stability and is thus the default and most commonly used method. The underlying computations rely on LAPACK routines for numerical linear algebra. For other norms $p$, it requires explicitly computing the inverse of $A$ and thus is generally less numerically stable than the 2-norm computation based on SVD.

We assess both accuracy and computational efficiency using the following metrics:
\begin{itemize}
    \item For each test matrix $\mathbf{A}$, let $\kappa_1(\mathbf{A})$ denote the ground truth condition number and $\hat{\kappa}_1(\mathbf{A})$ the estimate produced by a given method. We define \emph{Logarithmic relative error:}
\begin{equation}
\mathrm{LRE}(\mathbf{A}) = \frac{\left| \log_{10} \hat{\kappa}(\mathbf{A}) - \log_{10} \kappa(\mathbf{A}) \right|}{\left| \log_{10} \kappa(\mathbf{A}) \right|}.
\label{eq:lre}
\end{equation}
This normalized variant accounts for the magnitude of the true condition number, with values reported as percentages. In practice, a small number is added, e.g. $10^{16}$, to the denominator to prevent \emph{zero-division protection}.



We aggregate these metrics across the test set via their mean and maximum:
\begin{equation}
\overline{\mathrm{LRE}} = \frac{1}{N_{\mathrm{test}}} \sum_{i=1}^{N_{\mathrm{test}}} \mathrm{LRE}(\mathbf{A}_i), \quad
\mathrm{LRE}_{\max} = \max_{1 \leq i \leq N_{\mathrm{test}}} \mathrm{LRE}(\mathbf{A}_i).
\label{eq:aggregated_metrics}
\end{equation}

Unlike standard relative error in linear space, the logarithmic metrics \eqref{eq:aggregated_metrics} is well-suited to condition numbers spanning multiple orders of magnitude. For instance, a linear relative error of $100\%$ for $\kappa_1 = 10^6$ ($\hat{\kappa}_1 = 2 \times 10^6$) translates to $\mathrm{LRE} \approx 30\%$, correctly reflecting that the estimate remains within the same order of magnitude.

\item  To evaluate computational efficiency, we measure wall-clock time for each method over $R = 4$ independent runs per test matrix, reporting:

\begin{enumerate}
\item \textbf{Mean inference time:}
\begin{equation}
T_{\mu}(\mathbf{A}) = \mathrm{mean}\left\{ t_1(\mathbf{A}), \ldots, t_R(\mathbf{A}) \right\},
\label{eq:median_time}
\end{equation}
where $t_r(\mathbf{A})$ is the time for the $r$-th run. The median is preferred over the mean to mitigate outliers from system-level noise. For task 1), the runtime for GNN also additionally includes the trivial computation of $\|\mathbf{A}\|_1$. 

\item \textbf{Speedup factor:}
\begin{equation}
S_{\mathcal{M}_1 \to \mathcal{M}_2}(\mathbf{A}) = \frac{T_{\mu}^{\mathcal{M}_1}(\mathbf{A})}{T_{\mu}^{\mathcal{M}_2}(\mathbf{A})},
\label{eq:speedup}
\end{equation}
quantifying how many times faster method $\mathcal{M}_2$ is compared to exact method $\mathcal{M}_1$ (performed by \texttt{numpy.linalg.cond(A)}). We primarily report $S_{\mathrm{Exact} \to \mathrm{GNN}}$ and $S_{\mathrm{Higham} \to \mathrm{GNN}}$. 
\end{enumerate}

For the GNN, we decompose the total time as
\begin{equation*}
T_{\mathrm{GNN}}(\mathbf{A}) =
\begin{cases}
T_{\mathrm{feature}}(\mathbf{A}) + T_{\mathrm{inference}}(\mathbf{A}) + T_{\mathrm{norm}}(\mathbf{A}), & \text{if } \hat{\kappa}(\mathbf{A}) = \|\mathbf{A}\| \cdot g(\mathbf{A}), \\
T_{\mathrm{feature}}(\mathbf{A}) + T_{\mathrm{inference}}(\mathbf{A}), & \text{if } \hat{\kappa}(\mathbf{A}) = g(\mathbf{A}).
\end{cases}
\end{equation*}
where $T_{\mathrm{feature}}$ is the time to construct the graph and extract features $\bm{\phi}(\mathbf{A})$, $T_{\mathrm{inference}}$ is the forward pass through the GNN, and $T_{\mathrm{norm}}$ is the overhead of computing $\|\mathbf{A}\|_1$ and applying \eqref{eq:final_estimate}. 
\end{itemize}

\subsection{Numerical Experiments}

Following the above formulation, we focus on the two schemes on prediction of (1) $g(\mathbf{A}) \approx \|\mathbf{A}^{-1}\|_p$ and (2) $g(\mathbf{A}) \approx \kappa_p(\mathbf{A})$, $\forall \mathbf{A} \in \mathcal{M}$ and $p=1, 2$. To mitigate the effects of environmental variability, we independently reran all competing methods for each scheme. We run each competing algorithm four times while deprecating the first warm-up time, and take their average performance as the recorded score. 

\subsubsection{Parameterized Matrix Families}


To ensure that the trained model generalizes across diverse sparse linear systems, we construct a heterogeneous training corpus that includes both symmetric and nonsymmetric matrices, covering five representative problem classes as described below.

\begin{itemize}

\item \textbf{Partial Differential Equation Discretizations}:
\begin{itemize}

\item \textbf{2D Poisson Equation.}
We consider the Poisson problem
\begin{equation*}
-\Delta u = f
\end{equation*}
on the unit square with homogeneous Dirichlet boundary conditions. The equation is discretized on a uniform $m \times m$ grid using the standard 5-point finite difference stencil on interior grid points ($n = m^2$ unknowns).  

The resulting matrix $\mathbf{A} \in \mathbb{R}^{n \times n}$ is sparse, symmetric positive definite (SPD), and corresponds to nearest-neighbor coupling in the grid. Up to a constant scaling factor, its diagonal entries are equal to $4$ and off-diagonal entries are $-1$.  

Its condition number satisfies $\kappa_2(\mathbf{A}) = \Theta(m^2) = \mathcal{O}(h^{-2})$~\cite{trefethen1997numerical}.

\item \textbf{Anisotropic Diffusion Equation.}
We consider the strongly orthotropic anisotropic diffusion problem
\begin{equation*}
-\nabla \cdot (K \nabla u) = f,
\qquad K = \operatorname{diag}(\varepsilon, 1),
\end{equation*}
where $\varepsilon \ll 1$ induces strong anisotropy along the coordinate axes (principal directions aligned with the grid). The operator is discretized using a standard 5-point finite difference scheme on a uniform grid.

The resulting matrix remains sparse and SPD, while its conditioning deteriorates significantly as $\varepsilon \to 0$.  In our dataset, $\varepsilon$ is sampled log-uniformly as  $\varepsilon \sim 10^{\mathcal{U}(-8,-3)} $, where  $\mathcal{U}(a,b) $ denotes the uniform distribution on  $[a,b] $, to produce highly anisotropic systems.

\item \textbf{High-Contrast Diffusion Problems.}
We consider variable-coefficient diffusion operators of the form
\begin{equation*}
-\nabla \cdot (\kappa(x,y)\nabla u),
\end{equation*}
where the coefficient $\kappa(x,y)$ exhibits strong spatial variation.  

The operator is discretized on a uniform grid using a conservative 5-point scheme with harmonic averaging of $\kappa$ at cell interfaces, ensuring symmetry and positive definiteness of the resulting matrix.  

At each grid point, coefficients are sampled as $\kappa_i \in [1, 10^c]$ with $c \sim \mathcal{U}(6,13)$, producing matrices with large coefficient contrast. A uniform scaling factor (e.g., $h^{-2}$) is omitted, as it does not affect the spectral properties or condition number of the system.  

These matrices are representative of challenging problems arising in heterogeneous media and are typically highly ill-conditioned.

\item \textbf{2D Convection--Diffusion Equation.}  
We consider the steady-state convection--diffusion problem
\begin{equation*}
-\varepsilon \nabla^2 u + \boldsymbol{\beta} \cdot \nabla u = f, 
\qquad \boldsymbol{\beta} = (\beta_x, \beta_y),
\end{equation*}
on the unit square with homogeneous Dirichlet boundary conditions.  

The diffusion term is discretized using the standard 5-point stencil, while the convection term is discretized using a first-order upwind scheme. The resulting matrix $\mathbf{A} \in \mathbb{R}^{n \times n}$ is sparse with at most five nonzero entries per row.  

Due to the upwind discretization, the matrix is generally nonsymmetric whenever $\boldsymbol{\beta} \neq \mathbf{0}$. The scheme improves stability in convection-dominated regimes while explicitly breaking symmetry.
\end{itemize}

\item \textbf{Synthetic Random Matrices}:
\begin{itemize}

\item \textbf{Random Sparse SPD Matrices.}
We generate random sparse symmetric positive definite matrices by first sampling a sparse matrix $\mathbf{B}$ with independent entries and symmetrizing it:
\begin{equation*}
\mathbf{A}_0 = \mathbf{B} + \mathbf{B}^\top.
\end{equation*}
Positive definiteness is enforced by adding a sufficiently large diagonal shift:
\begin{equation*}
\mathbf{A} = \mathbf{A}_0 + \alpha \mathbf{I},
\end{equation*}
where $\alpha$ is chosen based on the row-wise absolute sums to ensure strict diagonal dominance.  

An additional small diagonal perturbation is applied to influence the spectrum. This procedure yields sparse SPD matrices spanning a broad range of conditioning, from moderately to highly ill-conditioned regimes.

\item \textbf{Ill-Conditioned SPD Matrices via Diagonal Scaling.}
To further generate extremely ill-conditioned matrices, we apply diagonal congruence transformations of the form
\begin{equation*}
\mathbf{A} = \mathbf{S} \mathbf{D} \mathbf{S},
\end{equation*}
where $\mathbf{D}$ is a sparse SPD matrix constructed as above and $\mathbf{S} = \operatorname{diag}(s_i)$ is a diagonal scaling matrix with logarithmically distributed entries.  

This transformation preserves symmetry and positive definiteness while significantly amplifying spectral variation, producing matrices with very large condition numbers.

\item \textbf{Symmetric Tridiagonal Matrices.}
We include symmetric tridiagonal matrices
\begin{equation*}
\mathbf{A} = \operatorname{tridiag}(-\alpha, 2, -\alpha) \in \mathbb{R}^{n \times n},
\end{equation*}
where $\alpha \sim \mathcal{U}(0.1, 0.9)$. These matrices are strictly diagonally dominant and hence SPD. Their eigenvalues are given analytically by
\begin{equation*}
\lambda_k = 2 - 2\alpha \cos\left(\frac{k\pi}{n+1}\right), \quad k=1,\dots,n,
\end{equation*}
facilitating direct validation of spectral estimation methods.

\end{itemize}

\end{itemize}

We limit the variation in our training  to increments of 2,000 (i.e., from 1,000 to 3,000), because increasing the gap would necessarily require more data. However, in this paper, we focus on model performance when trained with a limited amount of data across different machines or GPU scales. To verify that our method is insensitive to discrepancies in matrix sizes between the training and test sets, we restrict the matrices for training and validation to dimensions in the range of 1,000 and 3,000, and the matrices for testing to 500 and 5,000. We shuffled the final collection with a fixed random seed and split it into training, validation, and testing sets. The matrices for training, validation, and testing are sampled with equal probability from each class.  Our dataset consists of $1{,}000$ training matrices, $100$ validation matrices, and $200$ testing matrices.  The generation routines ensure that the matrices used for training and testing exhibit a wide range of condition numbers. The distribution of  condition numbers for the training and testing set are shown in \figurename~\ref{fig:cond_dist}. The statistical information of the generated matrices for training, validation, and testing is shown in \tablename~\ref{tab:matrix_statistics}.

%
%
%
%
\begin{figure}[htbp]
    \centering
    \subfigure[Training set (including validation set)]{\includegraphics[width=0.45\linewidth]{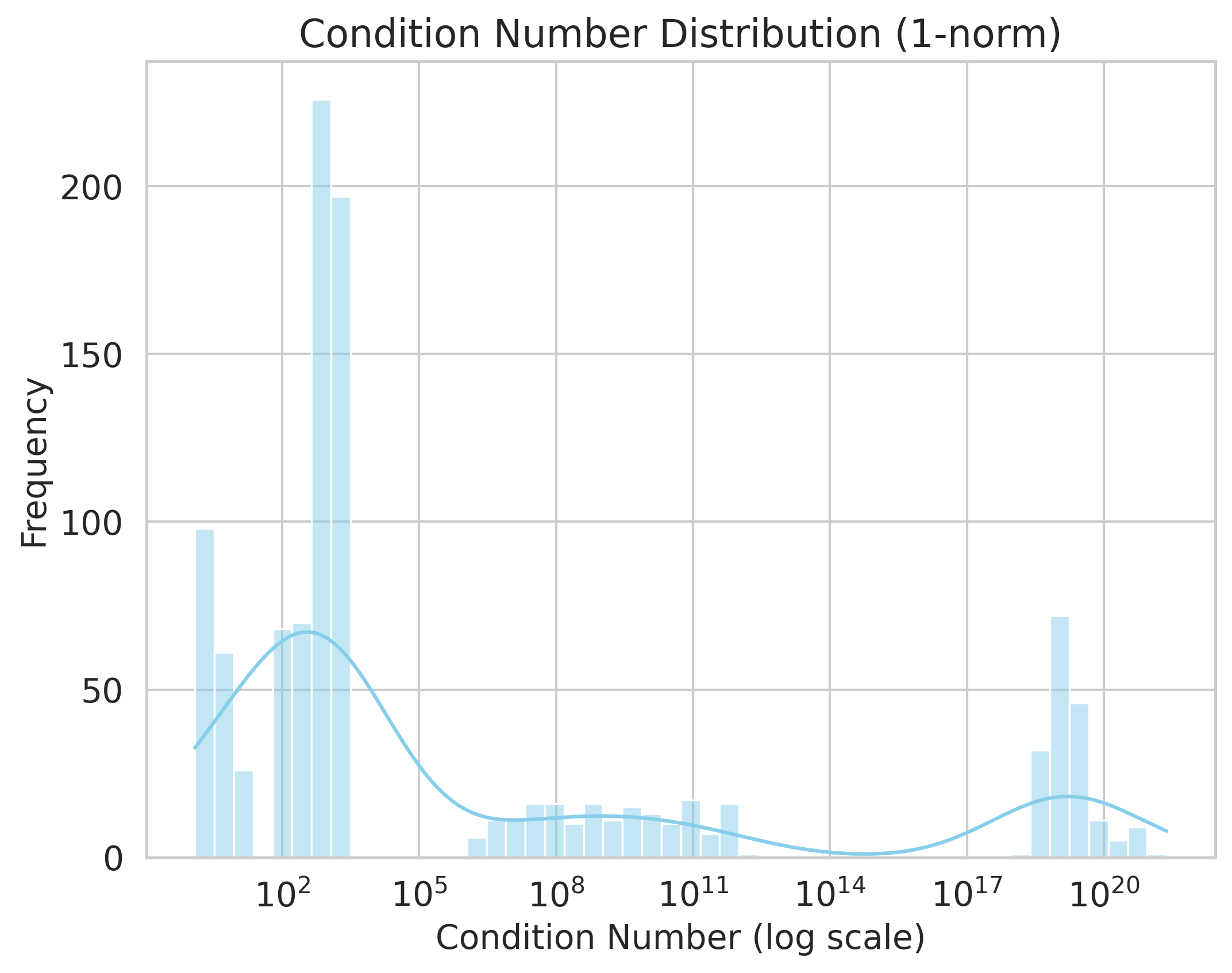}
    \includegraphics[width=0.45\linewidth]{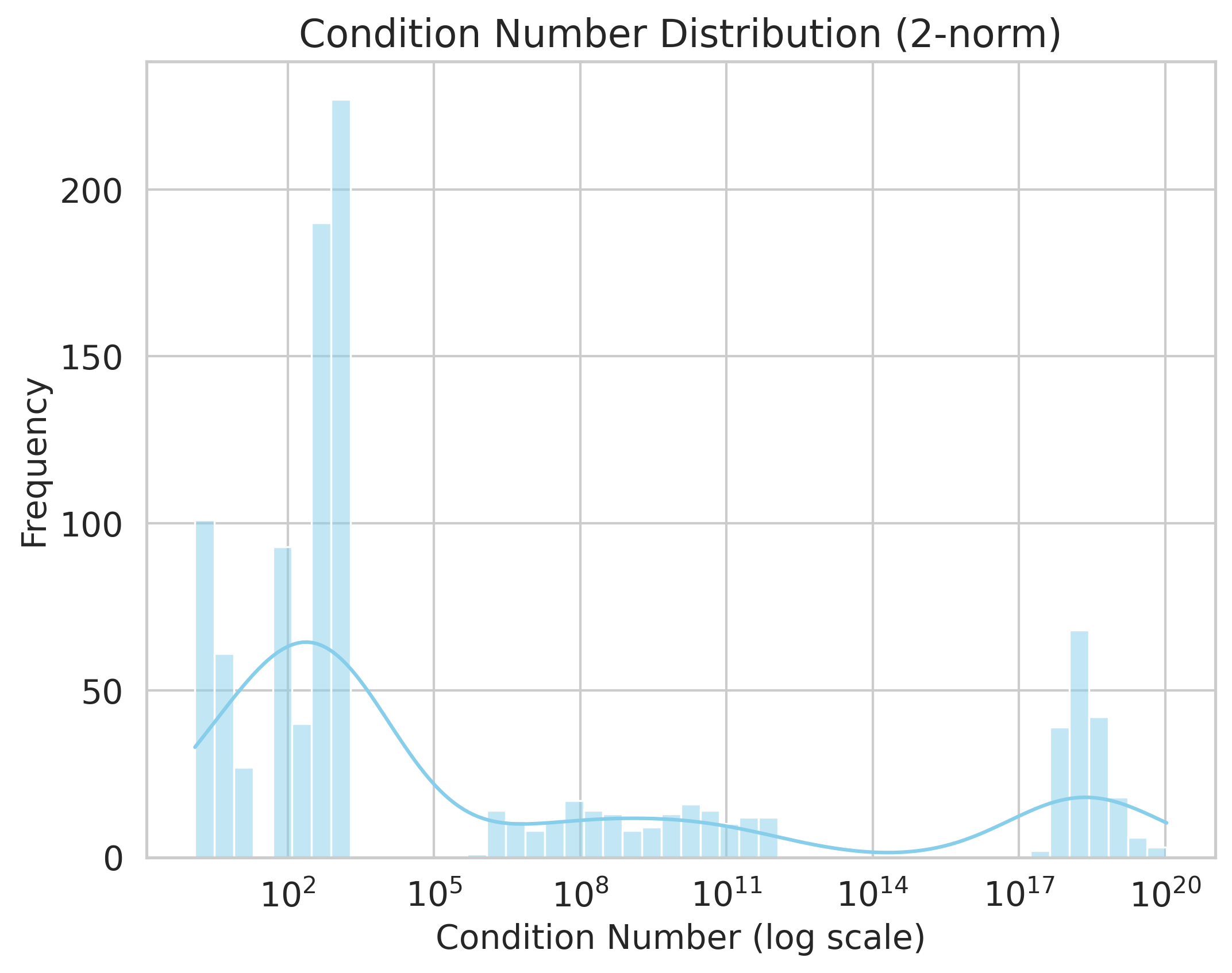}}
    \subfigure[Testing set.]{\includegraphics[width=0.45\linewidth]{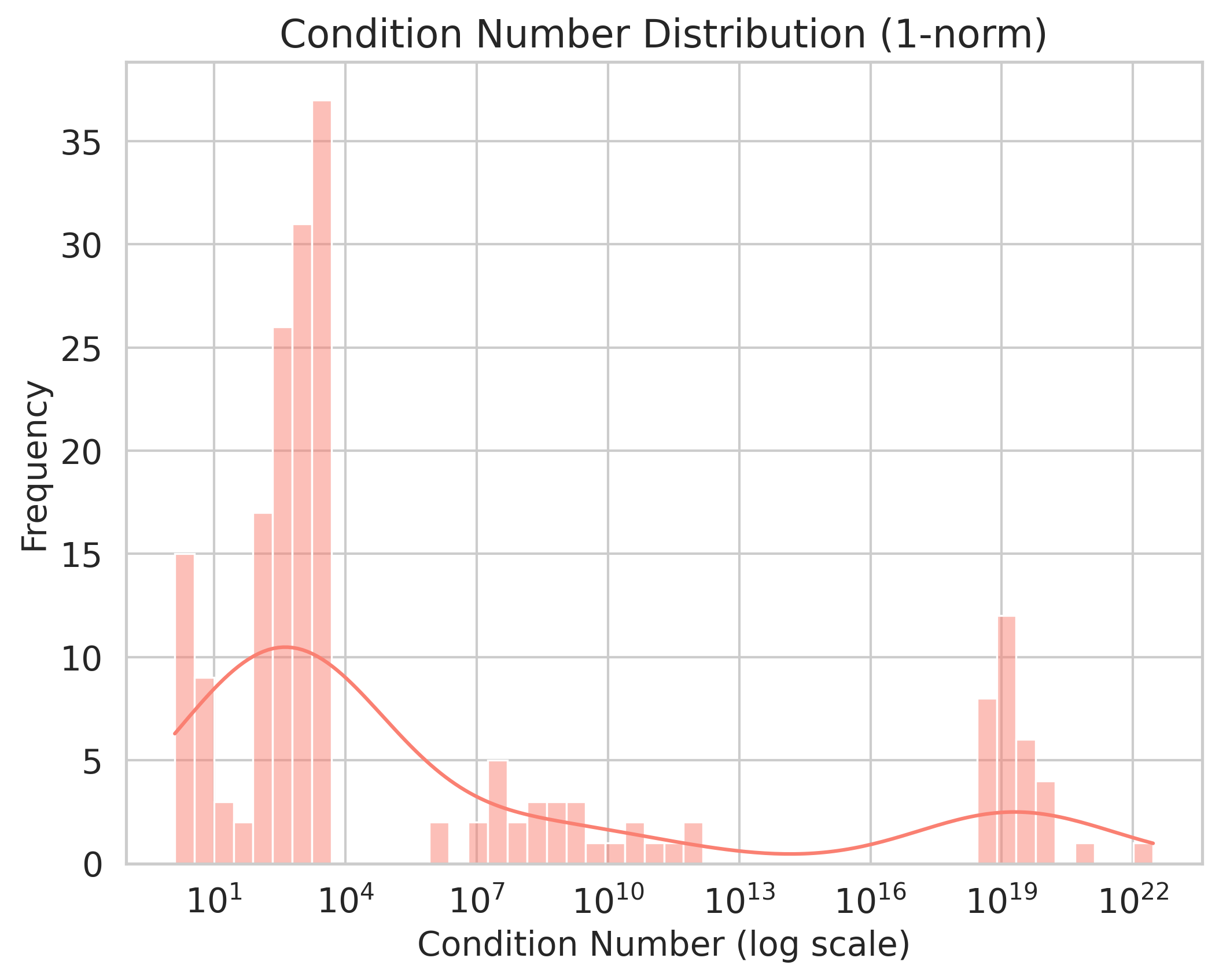}
    \includegraphics[width=0.45\linewidth]{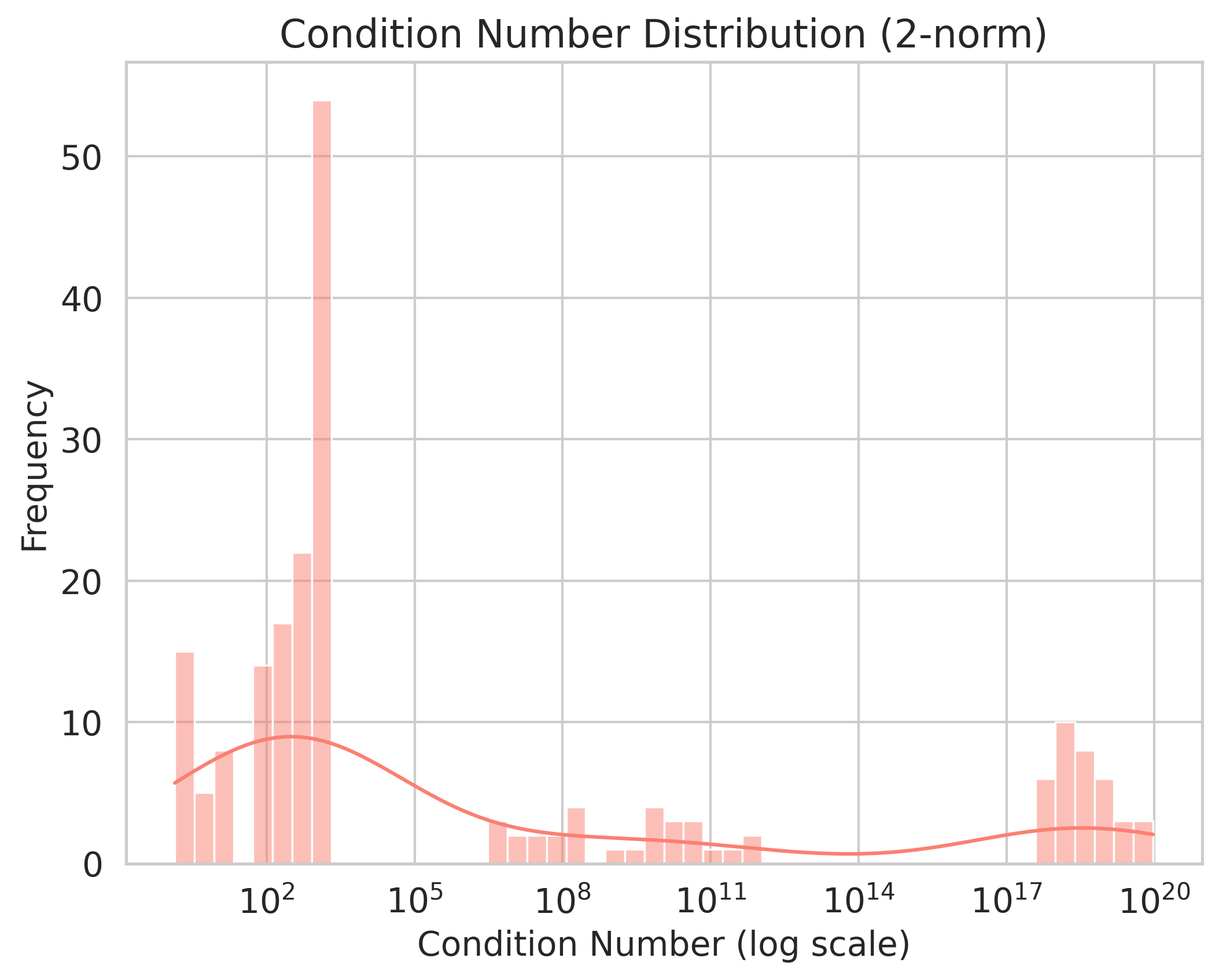}}
    \caption{Matrix condition number distribution on Parameterized Matrix Families test.}
    \label{fig:cond_dist}
\end{figure}

\begin{table}[htbp]
\centering
\caption{Statistics of the matrix datasets. Condition numbers are reported with respect to the 1-norm ($\kappa_1$) and 2-norm ($\kappa_2$). Sparsity is independent of the norm.}
\label{tab:matrix_statistics}
\resizebox{1\linewidth}{!}{%
\begin{tabular}{l c c c c c c}
\toprule
\multirow{2}{*}{Split} & \multirow{2}{*}{\# Matrices} & \multicolumn{2}{c}{Condition number ($\kappa_1$)} & \multicolumn{2}{c}{Condition number ($\kappa_2$)} & \multirow{2}{*}{Sparsity mean (range)} \\
\cmidrule(lr){3-4} \cmidrule(lr){5-6}
 & & Min & Max & Min & Max & \\
\midrule
Train & 1,000 & 1.23 & $2.38 \times 10^{21}$ & 1.26 & $1.04 \times 10^{20}$ & 96.68\% (80.88\%--99.90\%) \\
Test & 200 & 1.27 & $2.88 \times 10^{22}$ & 1.38 & $9.42 \times 10^{19}$ & 97.09\% (80.83\%--99.94\%) \\
Validation & 100 & 1.22 & $2.25 \times 10^{20}$ & 1.26 & $1.07 \times 10^{20}$ & 96.66\% (80.93\%--99.90\%) \\
\bottomrule
\end{tabular}
}
\end{table}


\begin{table}[ht]
\centering
\setlength{\tabcolsep}{2pt}
\caption{Performance comparison for $1$-norm and $2$-norm condition number computation. Time is reported in milliseconds (ms). Logarithmic relative error (LRE) and accuracy are reported in percentage (\%). The reported mean and maximum LRE are computed over samples with $\mathrm{LRE}<1$, while the accuracy metrics are computed over all samples. $t$: block size. S1/S2: Scheme 1/2. infer: inference.}
\label{tab:norm_comparison}
\resizebox{\linewidth}{!}{%
\begin{tabular}{llccccc}
\toprule
\multirow{2}{*}{Norm}
& \multirow{2}{*}{Method}
& \multirow{2}{*}{Mean Time (ms)}
& \multicolumn{2}{c}{Logarithmic Relative Error (\%)}
& \multicolumn{2}{c}{Accuracy (\%)} \\
\cmidrule(lr){4-5} \cmidrule(lr){6-7}
& &
& $\overline{\mathrm{LRE}}$
& $\mathrm{LRE}_{\max}$
& $\overline{\mathrm{LRE}}<0.5$
& $\overline{\mathrm{LRE}}<1$ \\
\midrule
\multirow{10}{*}{$1$-norm}
& \texttt{torch.cond} (CPU) & 411.83 & 0.00 & 0.15 & \textbf{100} & \textbf{100} \\
& \texttt{torch.cond} (GPU) & 43.89 & \textbf{0.00} & \textbf{0.00} & \textbf{100} & \textbf{100} \\
& Hager--Higham (CPU) & 156.48 & 0.00 & 0.15 & \textbf{100} & \textbf{100} \\
& Hager--Higham (GPU) & 40.61 & 0.01 & 0.19 & \textbf{100} & \textbf{100} \\
& Higham-SciPy (t=2, CPU) & 338.61 & 0.00 & 0.00 & \textbf{100} & \textbf{100} \\
& Higham-SciPy (t=10, CPU) & 384.78 & 0.00 & 0.00 & \textbf{100} & \textbf{100} \\
& GCN-S1 (CPU, resp. infer) & 181.20 (176.49) & 0.03 & 0.21 & \textbf{100} & \textbf{100} \\
& GCN-S1 (GPU, resp. infer) & \textbf{9.75 (6.03)} & 0.03 & 0.21 & \textbf{100} & \textbf{100} \\
& GCN-S2 (CPU, resp. infer) & 246.30 (241.56) & 0.47 & 0.98 & 35.0 & 69.0 \\
& GCN-S2 (GPU, resp. infer) & 10.15 (6.55) & 0.47 & 0.98 & 35.0 & 69.0 \\
\midrule
\multirow{8}{*}{$2$-norm}
& \texttt{torch.cond} (GPU) & 1228.14 & \textbf{0.00} & \textbf{0.00} & \textbf{100} & \textbf{100} \\
& Power Method (iter=10, GPU) & 195.67 & 0.08 & 0.53 & 97.5 & \textbf{100} \\
& Power Method (iter=30, GPU) & 471.60 & 0.08 & 0.53 & 97.5 & \textbf{100} \\
& Golub--Kahan (iter=10, GPU) & 20.10 & 0.60 & 0.91 & 17.5 & \textbf{100} \\
& Golub--Kahan (iter=30, GPU) & 67.01 & 0.49 & 0.89 & 57.5 & \textbf{100} \\
& SciPy \texttt{svds} (CPU) & 19066.24 & 0.00 & 0.00 & 21.5 & 21.5 \\
& GCN-S1 (GPU, resp. infer) & 22.24 (6.71) & 0.04 & 0.30 & \textbf{100} & \textbf{100} \\
& GCN-S2 (GPU, resp. infer) & \textbf{13.48 (6.56)} & 0.04 & 0.34 & \textbf{100} & \textbf{100} \\
\bottomrule
\end{tabular}%
}
\end{table}

For Scheme 1, where
$\hat{\kappa}(\mathbf{A}) = |\mathbf{A}| \cdot 10^{\tilde{g}(\mathbf{A}; \theta^)}$, the learning task becomes $\tilde{g}(\mathbf{A}; \theta^) \approx \log_{10} |\mathbf{A}^{-1}|$. The runtime and accuracy of the competing methods are reported in \tablename~\ref{tab:norm_comparison}.
Accordingly, on GPU, the GNN approach achieves consistently lower average runtimes than the Hager--Higham methods for 1-norm condition number estimation and than the exact approach by \texttt{torch.cond} for both 1-norm and 2-norm condition number estimation. On GPU, the average speedup of GNN over the Hager-Higham approach is close to $4\times$. For the 1-norm condition number, GNN under Scheme~1 achieves comparable $\overline{\mathrm{LRE}}$ and $\mathrm{LRE}_{\max}$ to the Hager-Higham method, whereas Scheme~2 leads to noticeably larger errors. For the 2-norm condition number, GNN under both prediction schemes achieve significantly lower  $\overline{\mathrm{LRE}}$ and $\mathrm{LRE}_{\max}$  compared to Golub--Kahan method. According to the runtime, GNN is faster than the 30-iteration Golub--Kahan method, and GNN under Scheme~2 is also faster than the 10-iteration variant, while achieving substantially smaller logarithmic relative errors. Across both norm types, GNN under Scheme~1 maintains $\overline{\mathrm{LRE}} < 0.5$ for all test samples; for the 2-norm condition number, this also holds for Scheme~2. The Hager-Higham implementations show a contrasting profile: the SciPy variants (Higham-SciP with $t=2, 10$) obtain smaller $\mathrm{LRE}_{\max}$ but require longer runtimes.

For Scheme 2, the task focuses on $\hat{\kappa}(\mathbf{A}) = 10^{\tilde{g}(\mathbf{A}; \theta^)}$,
the task becomes a straightforward precision where $\tilde{g}(\mathbf{A}; \theta^) \approx \log_{10} \bigl(|\mathbf{A}|,|\mathbf{A}^{-1}|\bigr)$. The corresponding runtime and accuracy are reported in \tablename~\ref{tab:norm_comparison}.
Similar to the results for the 1-norm condition number in Scheme 1, GNN again achieves the best runtime performance. For 1-norm and 2-norm condition number prediction, GNN attains $\overline{\mathrm{LRE}} < 0.5$ for all test samples, which is slightly better than the Hager–Higham method. For 2-norm condition number prediction, GNN achieves $\overline{\mathrm{LRE}} < 0.5$ for $100\%$ of the test samples, whereas the Golub–Kahan method achieves this threshold for only $57.5\%$ of samples with $10$ iterations and $21.5\%$ with $30$ iterations.
All competing methods maintain $\overline{\mathrm{LRE}} < 1$ except SciPy $\texttt{svd}$ across all test samples. 

Comparing the empirical results of Scheme 1 and Scheme 2, GNN exhibits a relatively smaller LRE and less runtime in Scheme~1 for both 1-and 2-norm cases. However, the 2-norm of a matrix is also a tricky task for estimating so Scheme~2 is more practical for 2-norm condition number estimation. Across the tests, the dominant bottleneck in the GNN approach lies in its inference phase. 

In both 1-norm prediction schemes on CPU, Higham-SciPy exhibits close runtime compared to \texttt{torch.cond} and twice the runtime achieved by the Hager-Higham method. While on GPU, \texttt{torch.cond} and Hager-Higham perform similarly in speed, while being much faster than Higham-Scipy, regardless of the parameter $t$. Among the competing algorithms on CPU deployment, the Hager-Higham achieves the fastest speed, followed by the GNN approaches.  With respect to both CPU and GPU deployment, we found that GNN achieves a significant decrease in runtime compared to competing algorithms on GPU.

In addition, from \tablename~\ref{tab:norm_comparison}, SciPy \texttt{svds} fails to achieve $\overline{\mathrm{LRE}} < 1$ over 85\%. It fails on the test mainly because it requires to compute the smallest singular value using $\texttt{svds}(A,\ \texttt{which="SM"})$, which is much harder and less stable than computing the largest singular value. If the matrix is ill-conditioned, then the smallest singular value \(\sigma_{\min}(A)\) is very close to zero, so tiny numerical errors can dominate the result. In addition, SciPy's sparse SVD method with the default ARPACK solver works through an eigenvalue problem related to $A^{T}A \text{ or } A^{T}$,
whose condition number satisfies approximately $\kappa(A^{T}A) = \kappa(A)^2$. Therefore, a matrix that is already badly conditioned becomes even more numerically difficult internally. As a result, the iterative solver may fail to converge, return an inaccurate estimate of \(\sigma_{\min}(A)\), or trigger an exception.
Thus, the failure is strongly related to the condition number, but it may also reflect solver instability rather than the matrix being exactly singular.

For the 2-norm prediction, it is observed that the Golub–Kahan method with increasing iterations achieves a significant increase in runtime, though with an observed error prone is decreased.


\subsubsection{SuiteSparse Matrix Collection}

We built our dataset from the SuiteSparse Matrix Collection. We focused on real, square, sparse matrices with dimensions between 2{,}000 and 10{,}000. After applying structural filters on size and sparsity, we converted each candidate matrix to compressed sparse row format. 

\begin{table}[htbp]
\centering
\caption{Statistics of the matrix datasets. Condition numbers are reported with respect to the 1-norm ($\kappa_1$) and 2-norm ($\kappa_2$). Sparsity is independent of the norm.}
\label{tab:matrix_statistics_suitesparse}
\resizebox{1\linewidth}{!}{%
\begin{tabular}{l c c c c c c}
\toprule
\multirow{2}{*}{Split} & \multirow{2}{*}{\# Matrices} & \multicolumn{2}{c}{Condition number ($\kappa_1$)} & \multicolumn{2}{c}{Condition number ($\kappa_2$)} & \multirow{2}{*}{Sparsity mean (range)} \\
\cmidrule(lr){3-4} \cmidrule(lr){5-6}
 & & Min & Max & Min & Max & \\
\midrule
Train & 180 & 1.00 & $3.82 \times 10^{19}$ & 1.00 & $1.45 \times 10^{19}$ & 99.34\% (87.52\%--99.98\%) \\
Test & 90 & 1.37 & $7.89 \times 10^{19}$ & 1.00 & $4.78 \times 10^{19}$ & 99.25\% (87.52\%--99.98\%) \\
Validation & 30 & 1.79 & $4.35 \times 10^{17}$ & 1.00 & $1.58 \times 10^{18}$ & 99.35\% (94.52\%--99.98\%) \\
\bottomrule
\end{tabular}
}
\end{table}

To keep only numerically well-behaved matrices, we removed any matrix for which the condition number estimation failed, returned non-finite values, or exceeded $10^{20}$. This eliminated cases of singular or extremely ill-conditioned data. Same as the last test,  we shuffled the final collection with a fixed random seed and split it into a 70\% training set and a 30\% test set. Eventually, we trained on 210 matrices (using 30 matrices among them for validation) and testing on 70 matrices. The information of training, validation, and testing set are presented in \tablename~\ref{tab:matrix_statistics_suitesparse}.  The distribution of condition numbers for the training and testing set are shown in \figurename~\ref{fig:cond_dist2}. 


For the ground-truth 2-norm condition number,  we computed the reference 1-norm condition number as $\kappa_1(A)=\|A|_1\|A^{-1}\|_1$, where $\|A\|_1$ was computed directly and $\|A^{-1}\|_1)$ was obtained by solving $AX=I$ and evaluating the 1-norm of the resulting inverse. For the  ground-truth 2-norm condition number, we computed the reference 2-norm condition number as $\kappa_2(A)=\sigma_{\max}(A)/\sigma_{\min}(A)$, where the singular values were obtained using dense singular value decomposition.

\begin{figure}[htbp]
    \centering
    \subfigure[Training set (including validation set).]{\includegraphics[width=0.45\linewidth]{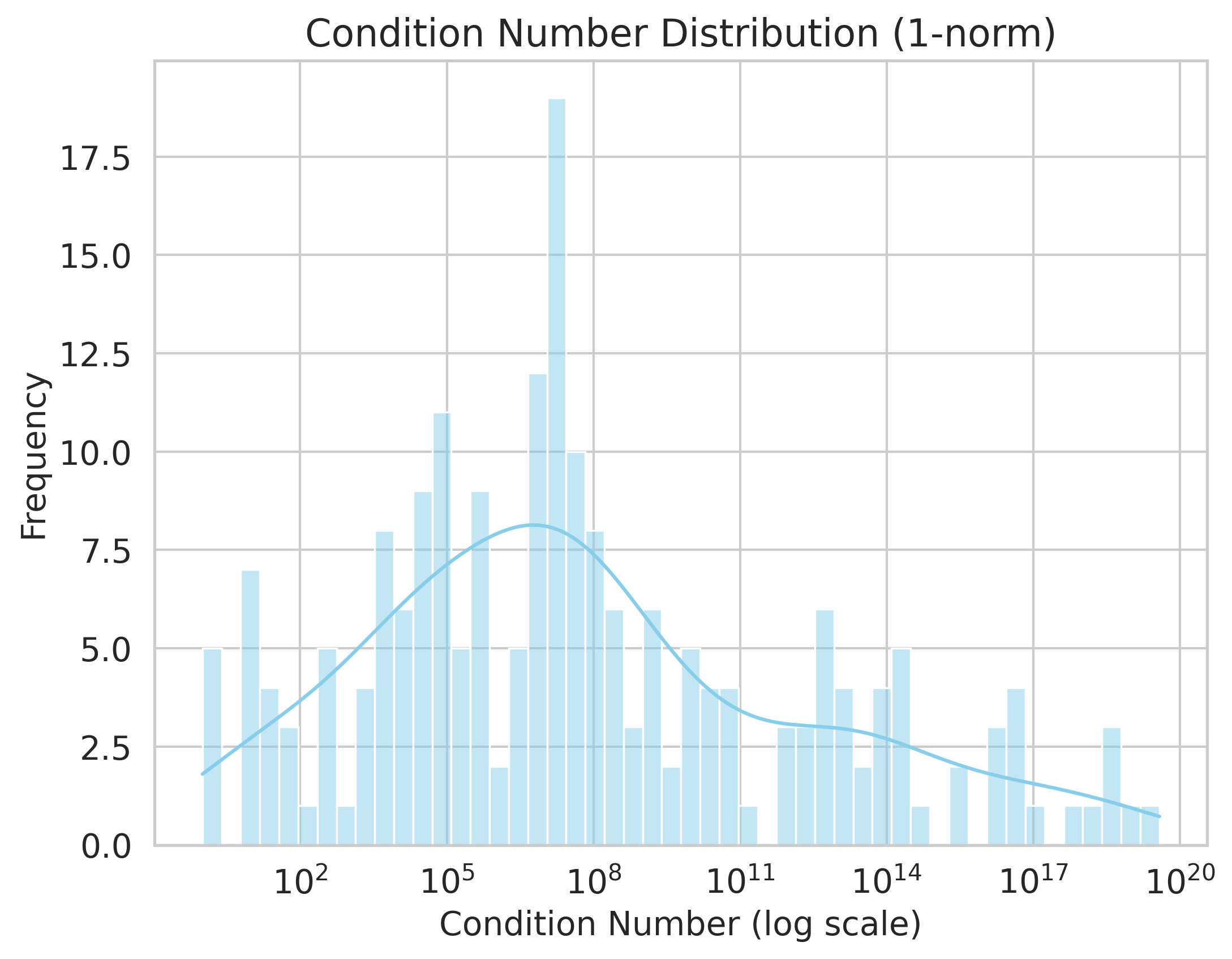}
    \includegraphics[width=0.45\linewidth]{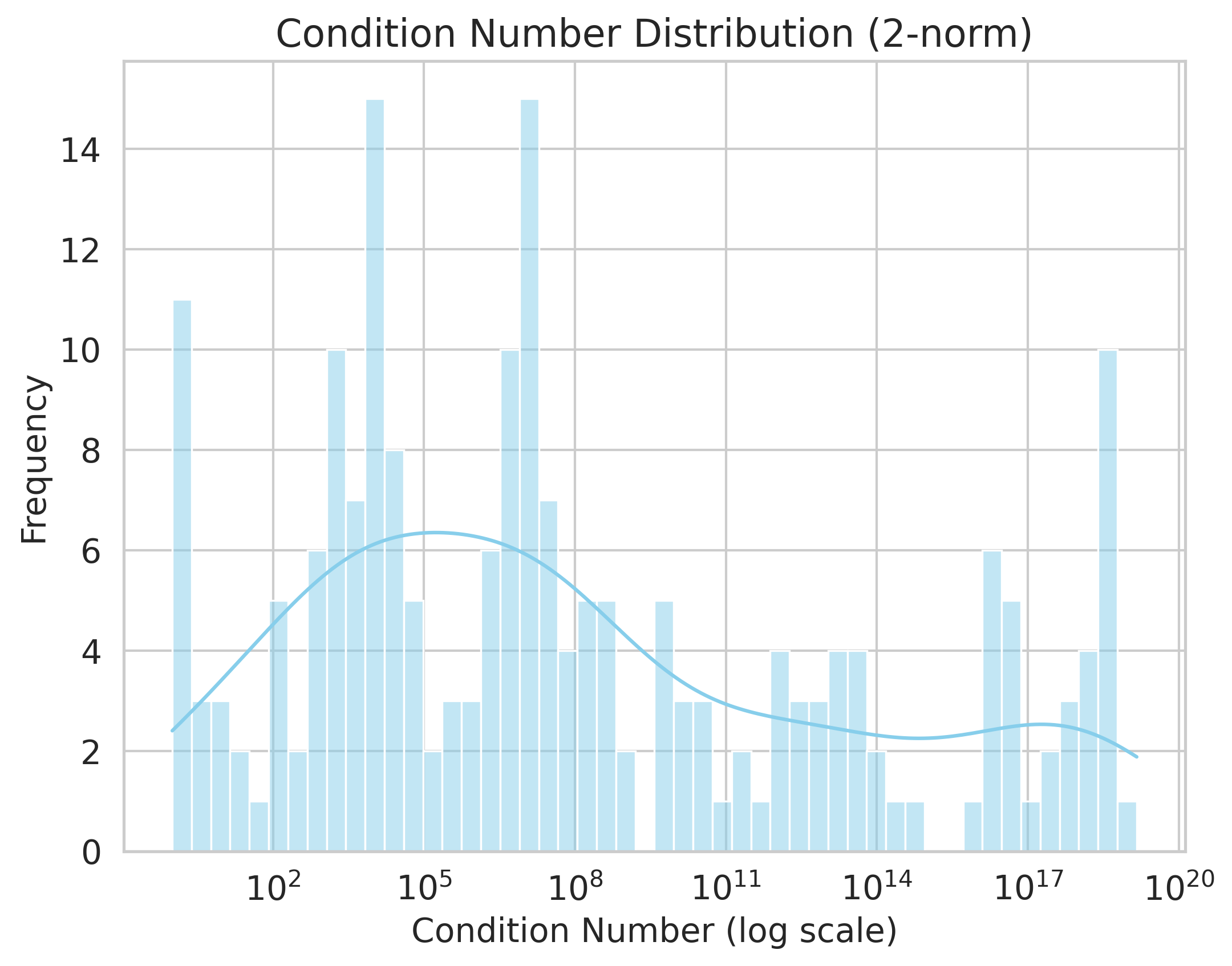}}
    \subfigure[Testing set.]{\includegraphics[width=0.45\linewidth]{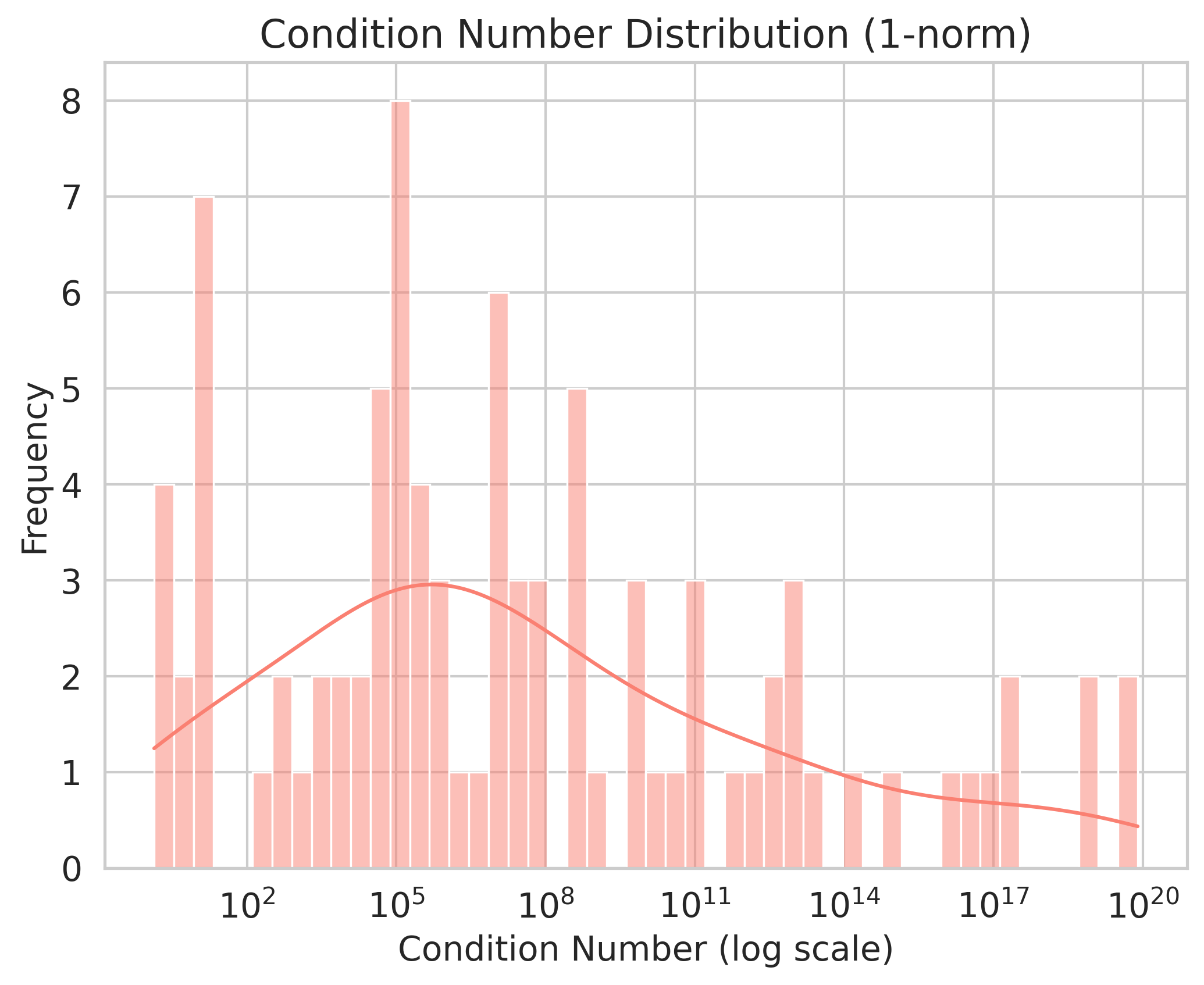}
    \includegraphics[width=0.45\linewidth]{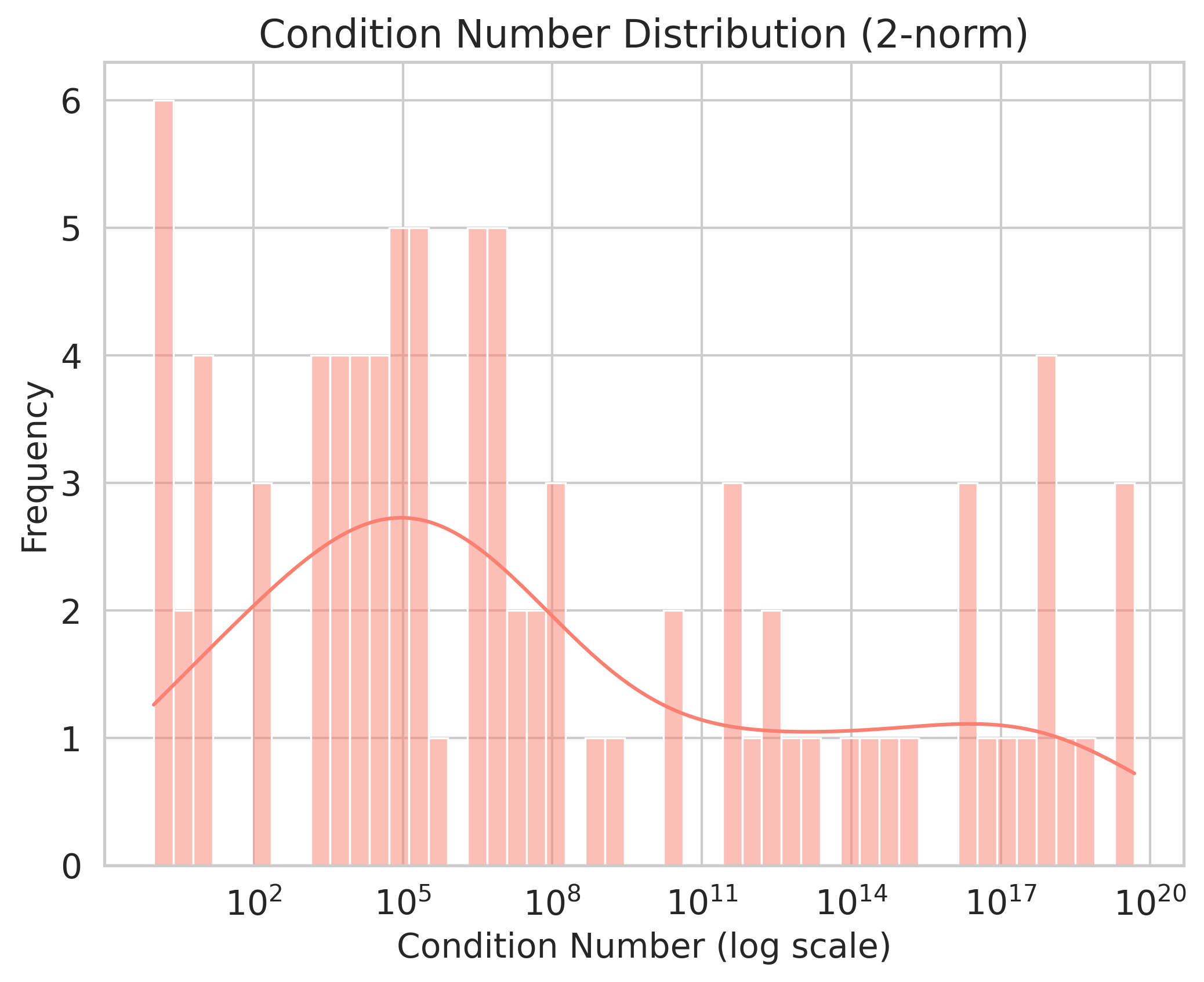}}
    \caption{Matrix condition number distribution on SuiteSparse Matrix Collection test.}
    \label{fig:cond_dist2}
\end{figure}

For the 1-norm condition number estimation, we compared our GNN estimators against several standard baselines: exact 1-norm computation in PyTorch, SciPy’s 1-norm estimator with different block sizes  ($t=2, t=10$), and a PyTorch implementation of the Hager–Higham estimator (both CPU and GPU versions when available). The GNN estimators follows the same two schemes--one predicting $\log_{10}\|A^{-1}\|_1$ and reconstructing $\kappa_1(A)$ from the known $\|A\|_1$, and another that directly predicts $\log_{10}\kappa_1(A)$. For the 2-norm condition number estimation, we benchmarked the proposed GNN estimators against several representative 1-norm and 2-norm baselines as illustrated in the last experiment. In addition,  we evaluated two GNN prediction schemes:  one predicting $\log_{10}\|A^{-1}\|$ and reconstructing $\kappa(A)$ from the known $\|A\|$, and another that directly predicts $\log_{10}\kappa(A)$.

The empirical results on CPU and GPU are illustrated in Table~\ref{tab:suitesparse_norm_comparison}. Similarly to the last experiment, SciPy \texttt{svds} fall most of tests, with $\overline{\mathrm{LRE}}<1$ only of 12.2\%. The GNN achieves the highest speedup on both CPU and GPU for the 1-norm and 2-norm cases.  In terms of estimation accuracy evaluated by $\overline{\mathrm{LRE}}$, the GNN estimators achieve $\overline{\mathrm{LRE}} < 1$ in more than 80\% of cases and $\overline{\mathrm{LRE}} < 0.5$ in over 70\% of cases. Although these results fall behind the Hager-Higham and Higham-SciPy approaches on 1-norm case, they are slightly better than the Golub--Kahan approach and worse than the power method on 2-norm case. Nevertheless, the GNN retains clear advantages in computational speed. The degraded accuracy compared to our previous simulation may be attributed to the limited training data and the distribution shift between the training and test sets.


\begin{table}[ht]
\centering
\setlength{\tabcolsep}{2pt}
\caption{
Performance comparison for $1$-norm and $2$-norm condition number computation. Time is reported in milliseconds (ms). Logarithmic relative error (LRE) and accuracy are reported in percentage (\%). The reported mean and maximum LRE are computed over samples with $\mathrm{LRE}<1$, while the accuracy metrics are computed over all samples. $t$: block size. S1/S2: Scheme 1/2. infer: inference.}
\label{tab:suitesparse_norm_comparison}
\resizebox{\linewidth}{!}{%
\begin{tabular}{llccccc}
\toprule
\multirow{2}{*}{Norm} & \multirow{2}{*}{Method}
& \multirow{2}{*}{Mean Time (ms)} & \multicolumn{2}{c}{Logarithmic Relative Error (\%)}
& \multicolumn{2}{c}{Accuracy (\%)} \\
\cmidrule(lr){4-5} \cmidrule(lr){6-7}
& &
& $\overline{\mathrm{LRE}}$ & $\mathrm{LRE}_{\max}$
& $\overline{\mathrm{LRE}}<0.5$ & $\overline{\mathrm{LRE}}<1$ \\
\midrule
\multirow{10}{*}{$1$-norm}
& \texttt{torch.cond} (CPU) & 3080.03 & 0.00 & 0.06 & \textbf{100.0} & \textbf{100.0} \\
& \texttt{torch.cond} (GPU) & 179.73 & 0.00 & 0.05 & \textbf{100.0} & \textbf{100.0} \\
& Higham (CPU) & 945.48 & 0.01 & 0.13 & \textbf{100.0} & \textbf{100.0} \\
& Higham (GPU) & 139.78 & 0.01 & 0.13 & \textbf{100.0} & \textbf{100.0} \\
& Higham-SciPy (t=2, CPU) & 1084.69 & 0.01 & 0.14 & \textbf{100.0} & \textbf{100.0} \\
& Higham-SciPy (t=10, CPU) & 1196.07 & 0.00 & 0.15 & \textbf{100.0} & \textbf{100.0} \\
& GCN-S1 (CPU, resp. infer) & 211.33 (207.42) & 0.23 & 0.93 & 81.1 & 94.4 \\
& GCN-S1 (GPU, resp. infer) & \textbf{7.47 (4.68)} & 0.23 & 0.93 & 81.1 & 94.4 \\
& GCN-S2 (CPU, resp. infer) & 218.92 (215.21) & 0.25 & 0.97 & 76.7 & 88.9 \\
& GCN-S2 (GPU, resp. infer) & 7.47 (4.62) & 0.25 & 0.97 & 76.7 & 88.9 \\
\midrule
\multirow{8}{*}{$2$-norm}
& \texttt{torch.cond} (GPU) & 38855.83 & 0.01 & 0.52 & 98.9 & \textbf{100.0} \\
& Power Method (iter=10, GPU) & 457.16 & 0.11 & 0.99 & 93.3 & 98.9 \\
& Power Method (iter=30, GPU) & 1066.85 & 0.10 & 0.98 & 95.6 & 98.9 \\
& Golub--Kahan (iter=10, GPU) & 14.45 & 0.64 & 0.94 & 21.1 & \textbf{100.0} \\
& Golub--Kahan (iter=30, GPU) & 50.11 & 0.56 & 0.91 & 26.7 & \textbf{100.0} \\
& SciPy \texttt{svds} (CPU) & 19631.02 & 0.09 & 0.83 & 11.1 & 12.2 \\
& GCN-S1 (GPU, resp. infer) & 12.20 (3.64) & 0.25 & 0.93 & 70.0 & 85.6 \\
& GCN-S2 (GPU, resp. infer) & \textbf{6.68 (3.62)} & 0.34 & 1.00 & 62.2 & 83.3 \\
\bottomrule
\end{tabular}%
}
\end{table}

\section{Limitations}\label{sec:limitations}

Our approach relies on the quality of the training data: the closer the distributions of the training and test sets, the better the accuracy of the model. Ideally, the test data is expected to follow a distribution similar to that used to train the model; however, this is not strictly required, as we have evaluated on the test of SuiteSparse Matrix Collection, where training and testing dataset are different, and our GNN model’s generalization ability is demonstrated. Future studies will investigate which factors (e.g., matrix sizes or matrix element magnitude) are most sensitive and contribute to performance loss on testing data. In this paper, we focus on the formulation of our approach and demonstrate the capabilities of GNNs for condition number estimation. We have not optimized the model architecture, hyperparameters, or training routine, so readers may achieve further refined results by tuning the hyperparameters. 
Besides, as observed in the experiments, the dominant runtime for condition number prediction for GNN lies in the inference. Therefore, potential speedup can be achieved with ONNX \cite{onnxruntime}.

\section{Conclusion}\label{sec:concludion}

We proposed a rapid condition number estimation approach using graph neural networks for sparse linear systems. We propose two prediction schemes, one by decomposing the condition number into an exactly computable forward norm and a neural-network-predicted inverse norm, and another by directly predicting the matrix condition number. To fully leverage the graph neural network power for matrix-based prediction, the neural network architecture combines graph convolutional layers to capture sparsity patterns with global feature encoders for statistical properties, enabling effective generalization. Future work will explore optimized GNN architectures and end-to-end feature engineering.  Extensive experiments for the two schemes show that, with a limited amount of data, our method achieves orders of magnitude speedup over exact computation and compares favorably against the Hager-Higham and Golub–Kahan methods at a certain scale in terms of speed while maintaining competitive accuracy. Additionally, we study the selected matrix features theoretically and empirically for their contributions to model performance.  

We believe that our work can be beneficial to existing work on precision tuning (e.g., \cite{car26pre}), since precision usage is highly dependent on condition number and a faster estimation of it can significantly aid applications.

\begin{acks}
The authors thank Cunyi Li for careful proofreading of the manuscript.
\end{acks}

\bibliographystyle{ACM-Reference-Format}
\bibliography{references}

\end{document}